\documentclass{article}

\usepackage[final]{colm2026_conference}

\usepackage{graphicx}
\usepackage{amsmath,amssymb}
\usepackage{booktabs}
\usepackage{multirow}
\usepackage{hyperref}
\usepackage{url}

\usepackage{float}
\usepackage{subcaption}
\usepackage{wrapfig}

\usepackage{amsthm}
\newtheorem{proposition}{Proposition}

\usepackage{amsmath,amsfonts,bm}

\def\eqref#1{equation~\ref{#1}}

\def\1{\bm{1}}

\DeclareMathAlphabet{\mathsfit}{\encodingdefault}{\sfdefault}{m}{sl}
\SetMathAlphabet{\mathsfit}{bold}{\encodingdefault}{\sfdefault}{bx}{n}

\graphicspath{{imgs/}}

\usepackage[most]{tcolorbox}
\usepackage{listings}
\usepackage{xcolor}

\newcommand{\ph}[1]{\texttt{\{#1\}}}

\lstdefinestyle{promptstyle}{
  basicstyle=\ttfamily\small,
  breaklines=true,
  columns=fullflexible,
  keepspaces=true,
  showstringspaces=false
}

\newtcblisting{promptbox}{
  colback=gray!5,
  colframe=gray!60,
  boxrule=0.5pt,
  arc=2pt,
  left=1mm,right=1mm,top=1mm,bottom=1mm,
  listing only,
  listing options={style=promptstyle}
}

\title{Pareto-Optimal RTL Code Generation via Multi-Objective Reinforcement Learning with Large Language Models}

\author{%
\textbf{Koki Takeshita}\quad
\textbf{Kenji Uebayashi}\quad
\textbf{Hikaru Tomonari}\quad
\textbf{Masaya Tsunokake}\\
\textbf{Koichiro Yawata}\quad
\textbf{Tatsuya Sasaki}\quad
\textbf{Yusuke Ohtsubo}\quad
\textbf{Kota Dohi}\quad
\textbf{Yasuhiro Sogawa}\\
\textnormal{Research \& Development Group, Hitachi, Ltd.}\\
\textnormal{\textbf{Correspondence:} koki.takeshita.cx@hitachi.com}
}

\begin{document}

\maketitle

\begin{abstract}
Recent studies have proposed reinforcement learning with verifiable rewards (RLVR) methods that enable LLMs to generate high-quality RTL circuits from natural-language specifications.
Existing approaches optimize circuit PPA (Power, Performance, Area) by aggregating multiple objectives into a single scalar reward, making it difficult to simultaneously generate diverse Pareto-optimal solutions across different trade-offs such as low power, low latency, and small area.
However, in practical design, the relative importance of each objective varies depending on product requirements and design constraints; thus, obtaining a diverse and high-quality candidate set is beneficial in that designers can choose a preferred trade-off a posteriori.
In this work, we are the first to propose a learning method for LLM-based circuit generation that optimizes, not the quality of a single circuit, but the quality and diversity of the entire circuit set obtained through multiple generations in the PPA objective space.
The proposed method consists of (1) a generation-based training pipeline that reformulates coverage improvement of the entire Pareto solution set into a sequential circuit generation problem, and (2) a reward function based on Hypervolume Contribution (HVC) that measures how much a generated circuit contributes to the Pareto front.
Evaluation on the RTLLM v2 benchmark shows that the proposed method achieves +13.3 percentage points in relative Hypervolume over a baseline without PPA optimization and +14.4 percentage points over a single-objective reinforcement learning method, while also demonstrating superiority in circuit diversity.

\end{abstract}

\section{Introduction}

Recently, reinforcement learning with verifiable rewards (RLVR), such as GRPO, has emerged~\citep{shao2024deepseekmath}, and the performance of LLMs on problems requiring reasoning capabilities has improved dramatically.
Led by OpenAI's o1 and o3~\citep{openai2024o1,openai2025o3} and DeepSeek-R1~\citep{guo2025deepseekr1}, subsequent models have been continuously released~\citep{yang2025qwen3,teamolmo2025olmo3}.
RLVR enables the acquisition of advanced logical reasoning capabilities in tasks where the quality of answers can be automatically verified, such as mathematics, programming, and question answering.

This paradigm has also been applied to electronic design automation (EDA), particularly to the NL-Code task of generating hardware description languages (HDLs) such as Verilog~\citep{verilog_history} from natural-language specifications.
\citet{zhu2025qimengcodevr1,wang2025verireason} automatically verify whether a generated circuit conforms to its natural-language specification using behavioral simulators such as Icarus Verilog~\citep{IcarusVerilog}, and improve specification compliance by learning with binary verification results as rewards.
\citet{chen2025chipseekr1,hsiao2025buffalo} additionally estimate the PPA (Power, Performance, Area) of generated circuits using downstream circuit evaluation tools such as OpenROAD~\citep{OpenROAD}, aggregate PPA into a single scalar reward via multiplicative or weighted-sum scalarization, and perform reinforcement learning.
This enables generating not only specification-compliant but also higher-performance circuits.

However, these existing studies optimize the generation quality of a single circuit under the scalarization function used during training.
That is, trained models do not explicitly learn to obtain a solution set spanning different trade-offs such as low power, low latency, and small area.
On the other hand, in circuit design, the relative importance of objectives such as power, delay, and area varies depending on product requirements and design constraints, and trade-offs frequently arise among these objectives.
Therefore, obtaining a diverse set of Pareto-optimal candidates is beneficial because it allows designers to choose a preferred candidate a posteriori according to the desired trade-off~\citep{Dobes2013,Visan2022,Liu2013}.
Furthermore, diverse and high-quality circuit candidates may be useful for generating promising parent individuals in search methods that combine LLM-based generation with evolutionary algorithms, such as AlphaEvolve~\citep{AlphaEvolve}.
(Details on the benefits of diverse and high-quality RTL generation are discussed in Section~\ref{sec:benefit_of_rtl_pareto}.)

In this work, for LLM-based circuit generation, we propose a learning method that aims to improve the quality and diversity in the PPA objective space of the entire circuit set obtained through multiple generations, rather than maximizing a single scalar reward for a single circuit generation.
To this end, (1) we introduce a generation-based training pipeline that reformulates the set-level optimization problem of improving coverage of the entire Pareto set into a sequential decision-making problem of generating the next candidate to add, enabling learning via GRPO-based methods.
(2) We introduce a reward function based on Hypervolume Contribution (HVC) that encourages the generation of solution groups close to the ideal Pareto front.
For the novel problem setting of Pareto solution exploration with LLMs in circuit generation, we define an evaluation metric called relative Hypervolume (rHV) and evaluate on RTLLM v2, a standard benchmark for NL-Code tasks.
The proposed method achieves +13.3 percentage points in rHV over the baseline and +14.4 percentage points over the single-objective learning method.

\section{Background}

\subsection{Abstraction Levels and Analysis Methods of Hardware Description Languages}

In digital hardware design, circuits are described using hardware description languages (HDLs), which serve a role analogous to programming languages in software.
HDLs are broadly categorized by their level of abstraction, ranging from high-level behavioral descriptions that focus on what the circuit should do, to low-level structural descriptions that specify connections between logic gates or standard cells.
Among these, one of the central representations used in LLM-based circuit generation is Register Transfer Level (RTL), which describes data transfers between registers and clock-cycle behavior.
RTL is widely used in industry and is amenable to behavioral verification and performance evaluation through simulation~\citep{DeepRTL2}.
Representative languages include Verilog~\citep{verilog_history} and VHDL~\citep{vhdl_overview}.

Test code for verifying whether an RTL-designed circuit correctly satisfies its specification is also written in HDL.
This is called a testbench.
Specification verification is performed by executing the testbench on a dedicated behavioral simulator.
For Verilog, a representative behavioral simulator is Icarus Verilog~\citep{IcarusVerilog}.

To estimate the PPA (Power, Performance, Area) of an RTL-designed circuit, it must first be converted to a lower-level gate-level representation.
This conversion is called logic synthesis.
A representative logic synthesis tool is Yosys~\citep{WolfGlaser2013Yosys}.
After logic synthesis, physical design tools such as OpenROAD~\citep{OpenROAD} perform placement and routing to estimate PPA.

\subsection{Benefits of Diverse and High-Quality RTL Generation}\label{sec:benefit_of_rtl_pareto}

In multi-objective circuit design, the relative importance of objectives such as power, delay, and area varies depending on product requirements and design constraints.
Therefore, obtaining a diverse set of Pareto-optimal candidates is beneficial because it allows designers to choose a preferred candidate a posteriori according to the desired trade-off~\citep{Dobes2013,Visan2022,Liu2013}.
While similar trade-off exploration can also be performed using downstream logic synthesis and physical design tools, such evaluation loops are computationally expensive, requiring many synthesis runs, with each candidate taking minutes to hours; thus, iterative downstream evaluation is costly~\citep{Bai2023}.
\cite{AlphaEvolve} also argues that directly manipulating RTL to improve circuits has the advantage of performing optimization early in the design flow.
For this reason, generating diverse and high-quality circuits at the upstream RTL stage is practically useful in that promising alternatives can be identified before invoking costly downstream evaluation.

Furthermore, the diversity of RTL candidates may also be useful in search methods that combine code generation and improvement by trained LLMs with evolutionary algorithms, such as AlphaEvolve~\citep{AlphaEvolve}.
AlphaEvolve maintains a database of previously generated programs and their evaluation results, reusing them as parent candidates and sources of inspiration while generating next candidates.
Experiments have verified its effectiveness across diverse domains, and in particular, for a task targeting TPU arithmetic circuits, it demonstrated that directly modifying RTL can present optimization candidates early in the design flow.
As an important implication, \cite{AlphaEvolve} states that retaining candidates that are good under different criteria, rather than keeping only candidates that are good under a single criterion, can stimulate the generation of more diverse candidate solutions.
This insight suggests that diverse and high-quality RTL candidates may be valuable not only for human designers but also as a promising parent candidate pool for subsequent evolutionary search.

\section{Related Work}

\subsection{RLVR in Circuit Generation}

ChipSeek-R1~\citep{chen2025chipseekr1} proposed a reinforcement learning framework for RTL generation based on a hierarchical reward design combining syntactic correctness, functional correctness, synthesizability, and PPA.
For PPA rewards, it aggregates PPA values estimated using OpenROAD~\citep{OpenROAD} into a single reward via a multiplication-based scalarization function, achieving RTL code generation with PPA performance exceeding human designs.
BUFFALO~\citep{hsiao2025buffalo} applies GRPO to the buffer tree generation task, updating model parameters based on circuit quality metrics obtained from downstream physical design tool evaluation.
Specifically, it normalizes the timing improvement and area increase and aggregates them into a scalar reward via a weighted sum, which is then maximized during training.

All of these methods optimize single circuit generation under a single scalarization function, and do not explicitly learn to generate diverse and high-quality solution sets spanning different PPA trade-offs.
In contrast, this work aims to improve the quality and diversity of the entire circuit set obtained through multiple generations for a given specification in the PPA objective space.

\subsection{Evaluation of Multi-Objective Reinforcement Learning}

\citet{teoh2025generalization} proposed the Normalized Hypervolume Generalization Ratio (NHGR) for evaluating cross-environment generalization in multi-objective reinforcement learning.
NHGR is defined by normalizing each objective based on the element-wise minimum and maximum of the optimal Pareto front for each evaluation context, then dividing the normalized hypervolume of the approximate front obtained by the generalist by the normalized hypervolume of the optimal front.
This normalization is introduced to prevent the difficulty of comparison that arises when ordinary Hypervolume depends on the scale of objective values and the range of achievable Pareto fronts for each problem.
Moreover, by defining it as a ratio, it offers the advantage of intuitively interpreting how good the quality and diversity of the solution set are relative to the ideal Pareto front.
When the true Pareto front is unavailable, non-dominated value vectors from specialist agents trained independently for each context are aggregated as an approximation of the Pareto front.

Our setting shares the motivation with \citet{teoh2025generalization} but differs in the problem unit.
In \citet{teoh2025generalization}, the evaluation target is the ``environment context,'' and specialists can be trained for each context.
In our work, we target LLM-based circuit generation, with evaluation units being individual samples in the benchmark.
Thus, training a specialist LLM separately for each sample to construct a reference front is impractical.
Instead, we integrate solutions obtained from all models and all trials for each problem and use the resulting non-dominated set as a pseudo Pareto front.

\section{Proposed Method}

Figure~\ref{fig:overview} shows an overview of the proposed method.
The method is a generation-based training pipeline that alternates between ``reinforcement learning with HVC reward'' and ``dataset reconstruction.''
By sequentially learning to generate the next circuit to add given the circuit history obtained so far, the method improves the quality and diversity of the overall circuit set obtained in the end.
``Reinforcement learning with HVC reward'' provides the LLM, conditioned on circuit history, with Hypervolume Contribution (HVC) as a reward indicating how much the generated circuit contributes to the Pareto front relative to the existing circuit set, and trains it using reinforcement learning.
``Dataset reconstruction'' aggregates circuits and their evaluation results generated during training, merges them as new circuit history into the current training dataset, and constructs the next-generation training dataset.
This enables efficient learning of Pareto solution exploration capabilities while progressively extending circuit history.

\subsection{Reinforcement Learning with HVC Reward}

In this phase, the model takes as input a specification written in natural language and previously generated circuits along with their PPAs during training, and learns to output the next circuit to generate via reinforcement learning.
The dataset used consists of an initial dataset containing specifications written in natural language, corresponding testbenches, and ground-truth circuit examples, merged with circuit history extracted through ``dataset reconstruction.''
As a reward, we use Hypervolume Contribution (HVC), which indicates how much the Hypervolume increases in the objective function space when the circuit is added to the existing circuit set.
This enables the model to learn to generate circuits with large contributions to the Pareto front relative to the existing circuit set at each generation.
At test time, the model recursively adds generated circuits and their evaluation values to the circuit history while generating circuits.
Since the model is trained to generate circuits with the largest Hypervolume Contribution at each generation iteration, the quality and diversity of the overall circuit set obtained in the end are expected to improve.

\paragraph{Reinforcement Learning Algorithm}
We use GRPO (Group Relative Policy Optimization)~\citep{shao2024deepseekmath}, which generates multiple responses $\{y_i\}_{i=1}^{G}$ from the old policy $\pi_{\theta_{\mathrm{old}}}$ for each input $q$, computes advantages from relative rewards within the group, and updates the policy with clipping and KL regularization without a separate value function.

\begin{subequations}\label{eq:grpo}
\begin{gather}
\begin{aligned}
\mathcal{J}_{\mathrm{GRPO}}(\theta)
&= \mathbb{E}_{q,\, \{y_i\}_{i=1}^{G} \sim \pi_{\theta_{\mathrm{old}}}}
\Biggl[
\frac{1}{G}\sum_{i=1}^{G}\frac{1}{|y_i|}\sum_{t=1}^{|y_i|}
\Bigl(
\min\!\bigl(
r_{i,t}\,\hat{A}_{i,t},\;
\operatorname{clip}(r_{i,t},\,1{\pm}\epsilon)\,\hat{A}_{i,t}
\bigr) \\
&\qquad\qquad\qquad\qquad\qquad\qquad\qquad
-\beta\, D_{\mathrm{KL}}\!\bigl[\pi_{\theta}\,\|\,\pi_{\mathrm{ref}}\bigr]
\Bigr)
\Biggr]
\end{aligned} \label{eq:grpo_obj} \\[4pt]
r_{i,t} = \frac{\pi_{\theta}(y_{i,t} \mid q,\, y_{i,<t})}{\pi_{\theta_{\mathrm{old}}}(y_{i,t} \mid q,\, y_{i,<t})},\qquad
\hat{A}_{i,t} = \hat{A}_{i} = \frac{R_i - \mathrm{mean}(\{R_j\}_{j=1}^{G})}{\mathrm{std}(\{R_j\}_{j=1}^{G})} \label{eq:grpo_vars}
\end{gather}
\end{subequations}

\begin{figure}[t]
  \centering
  \includegraphics[width=0.9\linewidth]{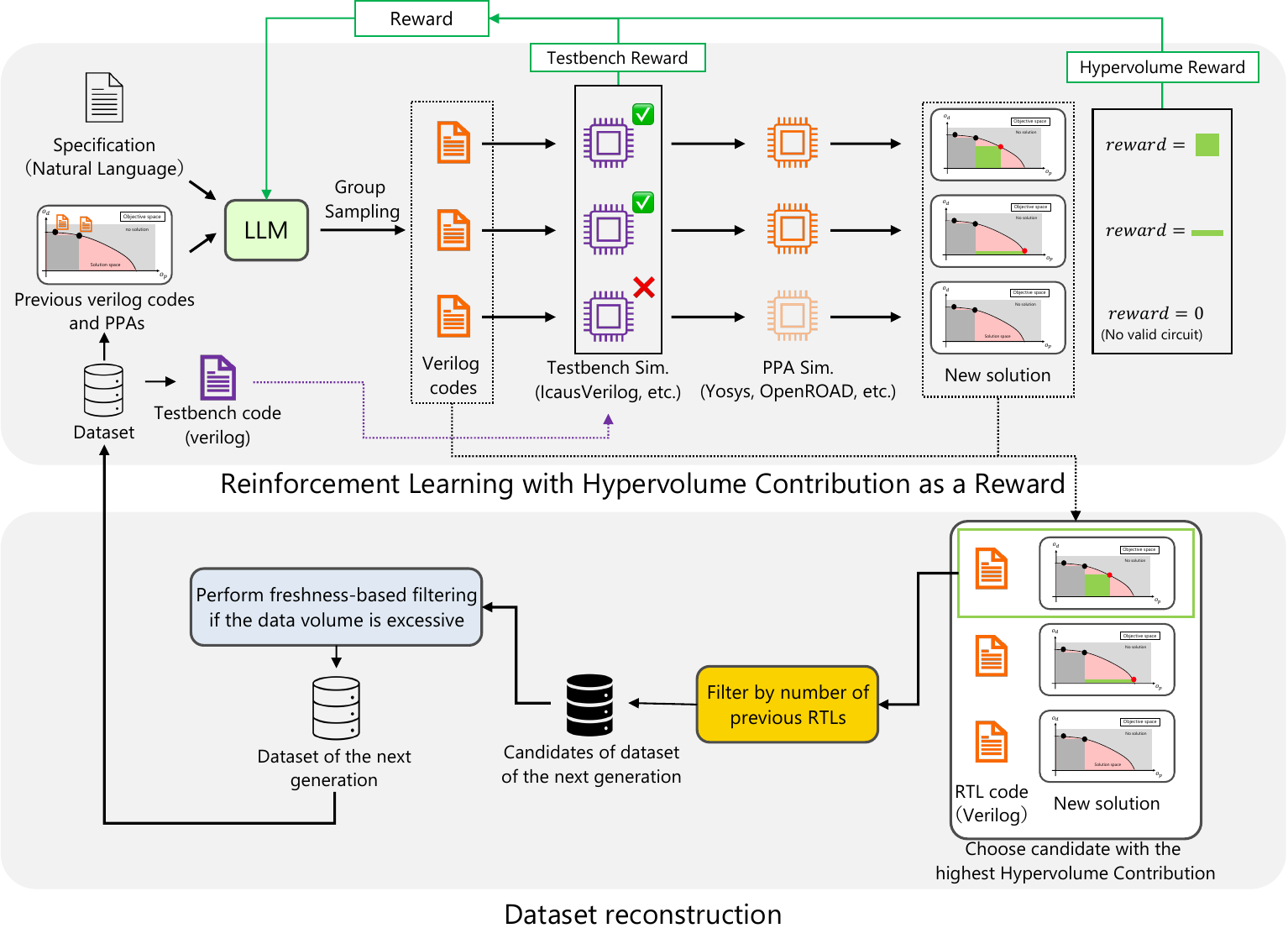}
  \caption{Overview of the proposed method. By alternating between reinforcement learning with HVC reward (top) and dataset reconstruction (bottom), the model learns Pareto solution exploration capabilities while progressively extending circuit history.}
  \label{fig:overview}
\end{figure}

\paragraph{Input Prompt}
The input prompt consists of circuit specification $s$ and circuit history $\mathcal{H}$.
The circuit history lists previously generated circuits and their PPA values.
This allows the model to obtain information for exploring new Pareto solutions different from existing solutions.
The actual prompt template is shown in Appendix~\ref{appendix:prompt}.

\paragraph{Hypervolume Contribution (HVC)}
To compute HVC, testbench verification is first performed on Icarus Verilog~\citep{IcarusVerilog} for each candidate generated through group sampling.
Only candidates that pass are subjected to PPA estimation using Yosys~\citep{WolfGlaser2013Yosys} and OpenROAD~\citep{OpenROAD}, following a similar approach to \cite{chen2025chipseekr1}.
Candidates failing the testbench or with invalid PPA receive $\mathrm{HVC}=0$.
For candidates with valid PPA, normalization based on the ground-truth circuit's PPA is performed to define objective function vectors.
The objective function vector for candidate $i$ is:
\begin{equation}
\mathbf{o}_i=(o_{P,i}, o_{D,i}, o_{A,i}),\quad
o_{P,i}=\frac{P_{\text{golden}}}{P_{\text{gen},i}+\epsilon},\;o_{D,i}=\frac{D_{\text{golden}}}{D_{\text{gen},i}+\epsilon},\;o_{A,i}=\frac{A_{\text{golden}}}{A_{\text{gen},i}+\epsilon}
\label{eq:objective}
\end{equation}
where $P_{\text{golden}}, D_{\text{golden}}, A_{\text{golden}}$ are the ground-truth circuit's PPA, $P_{\text{gen},i}, D_{\text{gen},i}, A_{\text{gen},i}$ are candidate $i$'s PPA, and $\epsilon$ is a small constant to avoid division by zero.
This normalization makes each component 1.0 for performance equal to the ground-truth circuit and greater than 1.0 when exceeding it.
Note that $P_{\text{gen},i}, D_{\text{gen},i}, A_{\text{gen},i}$ are all better when smaller.

Figure~\ref{fig:hypervolume} shows an example of multiple candidates and their HVCs in the objective function space.
HVC is defined as the increment in Hypervolume with origin as the reference point when adding candidate $i$:
\begin{equation}
\mathrm{HVC}(\mathbf{o}_i)=\mathrm{HV}(\{\mathbf{o}_j\mid j\le i\})-\mathrm{HV}(\{\mathbf{o}_j\mid j<i\})
\end{equation}
where $\mathrm{HV}$ is a function that computes the Hypervolume in the objective function space.
We use the implementation of \cite{Fortin2012DEAP} for computing Hypervolume.
Since the objective function $\mathbf{o}_i$ is a maximization metric, a larger HVC means that the candidate contributes more to the Pareto front.
Due to high reward variance, reward values are clipped to the range $[0,\, 3.0]$ for training stabilization.

The method uses reward $R_i$ as a combination of testbench verification reward and HVC.
Candidates passing the testbench receive a pass reward of $1.0$ plus HVC, while failing candidates receive $0$.
In summary, the variables in Equation~\ref{eq:grpo} are determined as follows:
\begin{equation}\label{eq:reward}
q = (s,\; \mathcal{H}),\quad
y_i \sim \pi_{\theta_{\mathrm{old}}}(\cdot \mid q),\quad
R_i = \begin{cases}
  1.0 + \mathrm{HVC}(\mathbf{o}_i) & \text{if testbench passed} \\
  0 & \text{otherwise}
\end{cases}
\end{equation}

\begin{wrapfigure}{r}{0.5\linewidth}
  \centering
  \vspace{-1em}
  \includegraphics[width=0.95\linewidth]{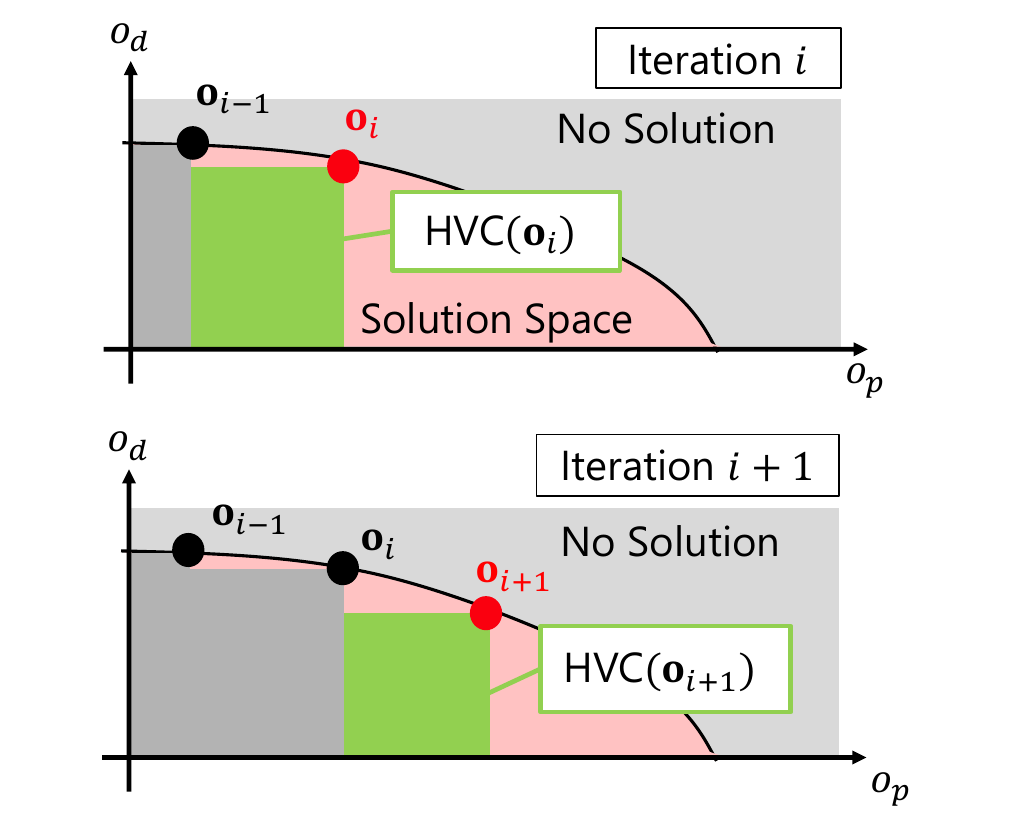}
  \caption{Example of Hypervolume Contribution (HVC). In the objective function space, the HVC is the Hypervolume increment with origin as the reference point when adding candidate $i$. A larger HVC means the candidate contributes more to the Pareto front.}
  \label{fig:hypervolume}
  \vspace{-1em}
\end{wrapfigure}

\subsection{Generation-based Dataset Reconstruction}

After each generation of training, circuits generated during training are aggregated to construct the next-generation training dataset.
This process progressively extends circuit history, enabling the model to learn Pareto solution exploration with longer histories as training progresses.

Specifically, for each candidate generated during training, the generated circuit, its PPA value, HVC, and the input circuit history are recorded.
Among circuits generated for the same original sample (i.e., sharing the same specification and circuit history), the candidate with the highest HVC is selected.
The selected circuit is then concatenated with the previous circuit history and added to the next-generation dataset.
To prevent excessive input length growth, samples with a history length of 5 or more are removed.
Additionally, to control dataset size, an upper limit of 1,000 samples per history length is imposed, with the oldest samples removed when this limit is exceeded.
Finally, the original dataset without circuit history is added to complete the next-generation training dataset.

Through this dataset reconstruction, starting from history length 0, samples with history lengths 1, 2, \ldots are added as generations progress.
By training on circuit histories of various lengths, the model acquires sequential Pareto solution exploration capabilities.

\section{Experiments}\label{sec:experiment}

\subsection{Baselines}

We evaluate all methods on two publicly available RTL-generation models as base models: \textbf{VeriReason-3B}~\citep{wang2025verireason} and \textbf{CodeV-R1-Distil-7B}~\citep{zhu2025qimengcodevr1}.
Using two base models with different training recipes allows us to verify that the effectiveness of the proposed method is not specific to a particular base model.
\textbf{VeriReason-3B} is a Qwen2.5-Coder-3B-based model trained through supervised fine-tuning on RTL generation data and GRPO reinforcement learning based on testbench feedback, but it does not directly optimize PPA objectives as rewards.
\textbf{CodeV-R1-Distil-7B} is a Qwen2.5-Coder-7B-Instruct-based reasoning-enhanced Verilog generation model that similarly emphasizes functional correctness without directly optimizing PPA objectives.
Both are thus positioned as existing strong RTL generation models that emphasize functional correctness.

On top of each base model, we compare the following five methods, designed from three perspectives to comprehensively verify the effectiveness of the proposed method:
(1) the unmodified base model without PPA optimization (Base) to evaluate the effect of PPA-aware training;
(2) single-scalar-reward reinforcement learning (Single-Objective PPA) to verify the superiority of HVC reward as multi-objective optimization;
(3) models using the same generation-based training pipeline as the proposed method but with different reward functions (Sequential TB, Sequential TB+PPA) to verify that generation-based training alone is insufficient and that combining it with HVC reward is essential.

\textbf{Base} is the unmodified base model (VeriReason-3B or CodeV-R1-Distil-7B) without any additional PPA-aware training, and is positioned as an existing strong RTL generation model that emphasizes functional correctness.
\textbf{Single-Objective PPA} is a model additionally trained with GRPO on the base model using testbench pass reward plus scalarized PPA reward similar to \citet{chen2025chipseekr1} (the formula is shown in Appendix~\ref{appendix:ppa_reward}), and corresponds to the single-objective reinforcement learning baseline in this work.
\textbf{Sequential TB} is a model trained using the same generation-based training pipeline as the proposed method on the base model, with only the reward function replaced by testbench pass reward.
It performs sequential generation conditioned on circuit history during evaluation, but PPA-based optimization or contributions to the Pareto front are not included in the reward.
\textbf{Sequential TB+PPA} is a model trained using the same generation-based training pipeline as the proposed method on the base model, with only the reward function replaced by the sum of testbench pass reward and PPA scalar reward.
That is, training is performed in a sequential generation setting, but the optimization target remains a single scalar reward without directly optimizing the quality or diversity of the entire solution set.
\textbf{Sequential HVC (Ours)} is a model trained using the proposed generation-based training pipeline on the base model with reward based on testbench pass reward and Hypervolume Contribution (HVC).
By learning to sequentially generate candidates with large contributions to the Pareto front relative to existing circuit history, the method aims to improve the quality and diversity of the entire solution set.

\subsection{Dataset}

Training data were constructed through two-stage filtering based on VeriReason's~\citep{wang2025verireason} hard level (1,149 problems).
First, 889 problems where the ground-truth circuit passed testbench verification and PPA estimation were extracted and used as training data for non-generation-based single-objective training (Single-Objective PPA).
Next, for these 889 problems, multiple circuit candidates were generated using CodeV-R1-Distil-7B~\citep{zhu2025qimengcodevr1}, and 78 problems where negative correlation ($r<0$) was observed between area and delay were extracted.
This is because problems where the Pareto front degenerates to a single point provide weak learning signals for multi-objective optimization.
These 78 problems were used for generation-based training (Sequential TB/Sequential TB+PPA/Sequential HVC).
For evaluation, 45 problems out of 50 in RTLLM v2~\citep{liu2024openllmrtl} where testbench verification and PPA estimation succeeded were used.

\subsection{Evaluation Method}

For each RTLLM v2 problem, circuits were generated by recursively adding generated circuits and their evaluation values to the circuit history, and evaluated using the following three metrics.
Evaluation was conducted with high-temperature sampling ($T=0.5$) with multiple trials per problem to reduce variance.

\paragraph{Relative Hypervolume (rHV)}
This metric is inspired by the NHGR of \citet{teoh2025generalization} described in Section~3.2.
In this work, for each problem $i$, we construct a \textbf{reference archive} $\Pi_i = \mathrm{ND}(U_i)$ as the non-dominated set of the union $U_i$ of solutions obtained from all models and all trials, and define rHV as follows:
\begin{equation}
  \mathrm{rHV}_{m,i,j} = \frac{\mathrm{HV}(\tilde{S}_{m,i,j}; \mathbf{r})}{\mathrm{HV}(\tilde{\Pi}_i; \mathbf{r})}
\end{equation}
where $S_{m,i,j}$ is the set of circuits generated by method $m$ on problem $i$ in trial $j$, $\tilde{S}_{m,i,j}$ and $\tilde{\Pi}_i$ are these solution sets mapped to the objective space normalized to $[0,\, 1]$ using the minimum and maximum PPA values in $U_i$, and $\mathbf{r}$ is the reference point.
rHV takes values in $[0,\, 1]$, with values closer to 1.0 indicating achievement of a Pareto front closer to the reference archive.
For aggregate evaluation, we report the mean across all problems and trials: $\bar{\mathrm{rHV}}_m$.
See Appendix~\ref{sec:rhv_details} for detailed mathematical definitions of rHV.

\paragraph{Win Rate}
For each problem $i$ and trial $j$, a win is assigned to the model achieving the highest Hypervolume among all compared models.
When multiple models achieve the same maximum Hypervolume, the win is split equally among them.
Formally, the Win Rate for model $m$ on problem $i$ and trial $j$ is defined as:
\begin{equation}
  \mathrm{WR}_{m,i,j} = \frac{\mathbf{1}[\mathrm{HV}_{m,i,j} = \max_{m'} \mathrm{HV}_{m',i,j}]}{|\{m' : \mathrm{HV}_{m',i,j} = \max_{m''} \mathrm{HV}_{m'',i,j}\}|}
\end{equation}
where $\mathbf{1}[\cdot]$ is the indicator function.
For aggregate evaluation, we report the mean across all problems and trials: $\bar{\mathrm{WR}}_m = \frac{1}{n_d n_t} \sum_{i,j} \mathrm{WR}_{m,i,j}$.
Win Rate takes values in $[0,\, 1]$, with larger values indicating greater superiority over other models.
Since Win Rate considers only relative wins within each problem, it can be directly aggregated without normalization even when Hypervolume scales vary significantly across problems.

\paragraph{Circuit Diversity}
The number of circuits with unique PPA values generated per problem and trial.
Higher values indicate that more diverse circuits can be generated in a single trial.

\section{Results}

\subsection{Relative Hypervolume Comparison}

\begin{table}[t]
\centering
\caption{Relative Hypervolume (rHV) for each model. Each column shows the rHV measured after a given number of sequential generation iterations (\textit{iter}), where the model generates a new circuit conditioned on the circuit history accumulated so far.}
\label{tab:rhv_results}
\begin{tabular}{llccc}
\toprule
\textbf{Base Model} & \textbf{Method} & \textbf{rHV (iter=1)} & \textbf{rHV (iter=5)} & \textbf{rHV (iter=10)} \\
\midrule
\multirow{5}{*}{VeriReason-3B} & Base & 13.1\% & 17.5\% & 19.1\% \\
& Single-Objective PPA & 10.4\% & 15.4\% & 18.0\% \\
& Sequential TB & 15.5\% & 18.5\% & 19.2\% \\
& Sequential TB+PPA & 17.2\% & 20.6\% & 21.6\% \\
& Sequential HVC (Ours) & \textbf{22.9\%} & \textbf{30.7\%} & \textbf{32.4\%} \\
\midrule
\multirow{5}{*}{CodeV-R1-Distil-7B} & Base & 23.1\% & 34.1\% & 38.4\% \\
& Single-Objective PPA & 23.5\% & 34.5\% & 37.8\% \\
& Sequential TB & \textbf{28.4\%} & 32.9\% & 33.5\% \\
& Sequential TB+PPA & 27.7\% & 31.2\% & 32.0\% \\
& Sequential HVC (Ours) & 26.7\% & \textbf{40.6\%} & \textbf{43.2\%} \\
\bottomrule
\end{tabular}
\end{table}

Table~\ref{tab:rhv_results} shows the relative Hypervolume (rHV) for each model.
The proposed method (Sequential HVC) achieved the highest rHV at iter=5 and iter=10.
With VeriReason-3B as the base model, at iter=10, it showed 32.4\% compared to Base's 19.1\% (+13.3 percentage points) and Single-Objective PPA's 18.0\% (+14.4 percentage points).
With CodeV-R1-Distil-7B as the base model, at iter=10, it likewise showed 43.2\% compared to Base's 38.4\% (+4.8 percentage points) and Single-Objective PPA's 37.8\% (+5.4 percentage points), confirming that the effectiveness of the proposed method does not depend on the choice of base model.

\subsection{Win Rate Comparison}\label{sec:win_rate_comparison}

Figure~\ref{fig:win_rate} shows the Win Rate progression per generation count (iteration).
The proposed method (Sequential HVC) achieved the highest Win Rate at almost all iterations in both base models; the only exception is iter=1 for CodeV-R1-Distil-7B, where Sequential TB is marginally higher.
As the generation count increases, its Win Rate rises steeply and its lead over the other methods widens.
This indicates that the proposed method not only finds good circuits initially but also continues exploration using circuit history to acquire more diverse Pareto-optimal solutions.

Notably, the Win Rate of the proposed method is maintained at a high level even beyond iter=5, the maximum circuit-history length seen during training, rather than degrading for longer exploration.
This confirms that the proposed method generalizes to exploration horizons longer than those used in training.

Conversely, the Win Rate of Sequential TB and Sequential TB+PPA declines as the number of iterations increases, leaving them with the lowest Win Rate at later iterations in both base models.
This suggests that with testbench reward or scalar PPA reward alone, reward hacking occurs where the model simply regenerates previously generated circuits, so that additional iterations fail to explore new Pareto solutions.

\subsection{Circuit Diversity Comparison}

\begin{figure}[t]
  \centering
  \begin{subfigure}[t]{0.495\linewidth}
    \centering
    \includegraphics[width=\linewidth]{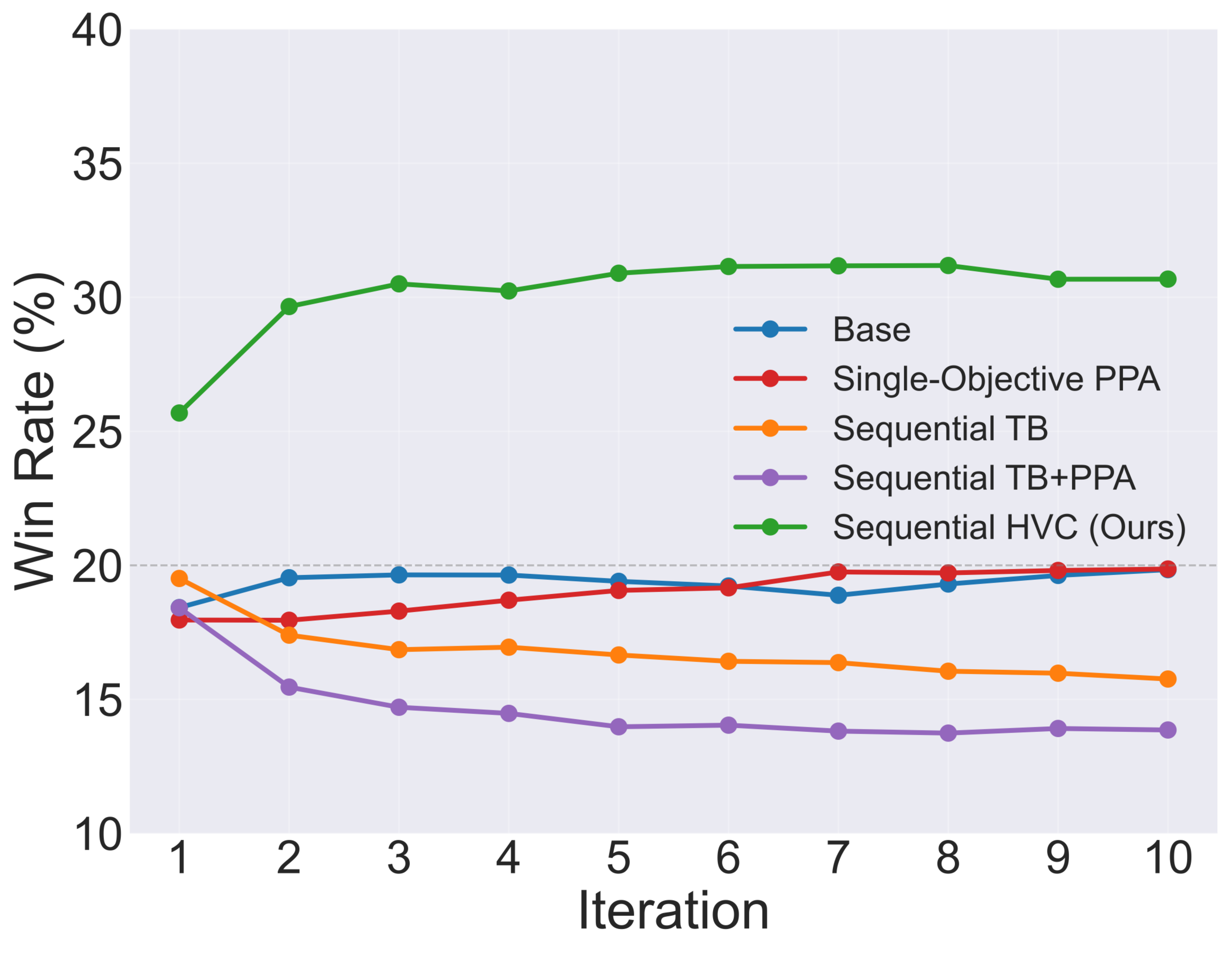}
    \caption{VeriReason-3B}
    \label{fig:win_rate_verireason}
  \end{subfigure}
  \hfill
  \begin{subfigure}[t]{0.465\linewidth}
    \centering
    \includegraphics[width=\linewidth]{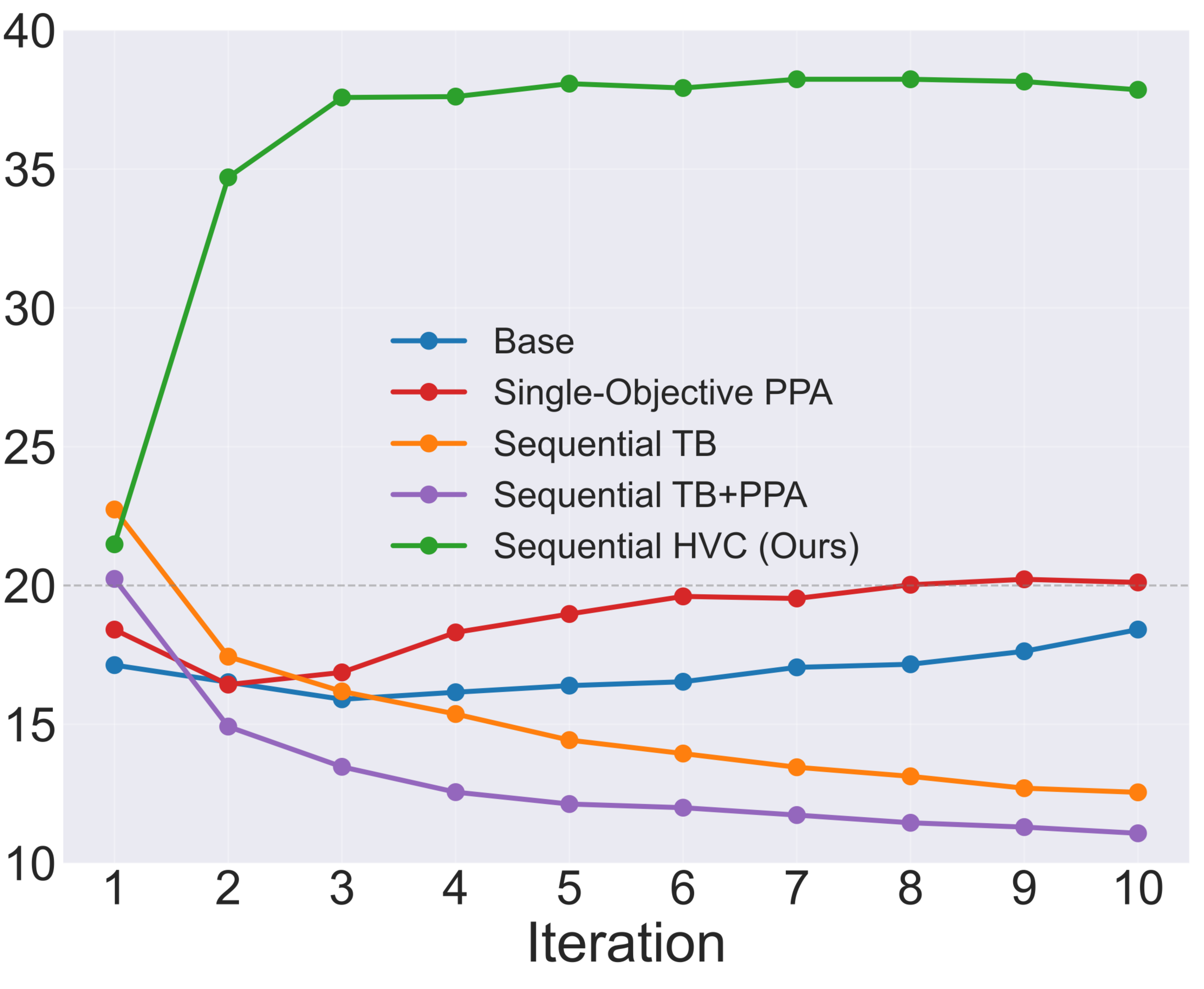}
    \caption{CodeV-R1-Distil-7B}
    \label{fig:win_rate_codev}
  \end{subfigure}
  \caption{Win Rate progression per generation count.}
  \label{fig:win_rate}
\end{figure}

\begin{figure}[t]
  \centering
  \begin{subfigure}[t]{0.404\linewidth}
    \centering
    \includegraphics[width=\linewidth]{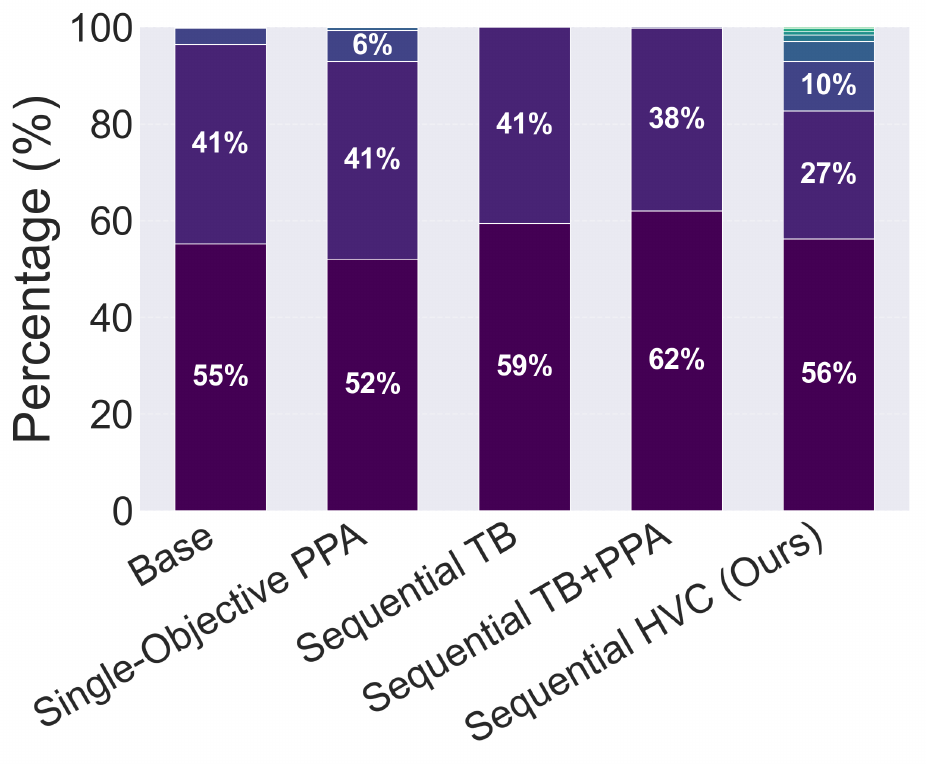}
    \caption{VeriReason-3B}
    \label{fig:circuit_diversity_verireason}
  \end{subfigure}
  \hfill
  \begin{subfigure}[t]{0.556\linewidth}
    \centering
    \includegraphics[width=\linewidth]{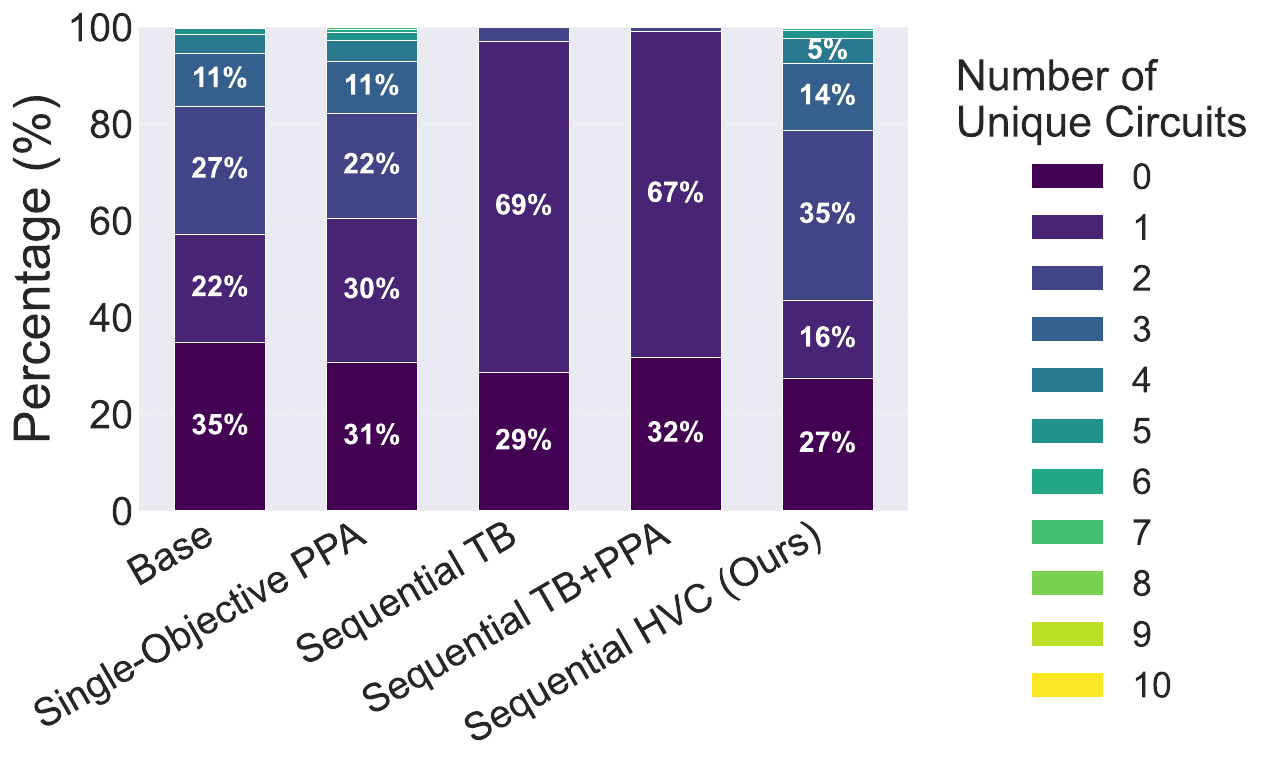}
    \caption{CodeV-R1-Distil-7B}
    \label{fig:circuit_diversity_codev}
  \end{subfigure}
  \caption{Circuit diversity comparison.}
  \label{fig:circuit_diversity}
\end{figure}

Figure~\ref{fig:circuit_diversity} shows the circuit diversity comparison.
The proposed method (Sequential HVC) was confirmed to generate the most diverse circuits in both base models.
Base and Single-Objective PPA show intermediate diversity, with their relative ordering depending on the base model.

On the other hand, Sequential TB and Sequential TB+PPA exhibit the lowest circuit diversity in both base models, frequently producing only a single unique circuit per trial.
This directly evidences the reward hacking observed in the Win Rate results in Section~\ref{sec:win_rate_comparison}: with testbench reward or scalar PPA reward alone, the model tends to regenerate previously generated circuits, suppressing exploration of new circuits.
This underscores that explicitly incorporating contributions to the Pareto front into the reward, as in HVC, is essential in the proposed generation-based pipeline.
To further confirm that the gains of the proposed method do not merely stem from avoiding duplicate generations, we conduct an ablation that adds an explicit repetition penalty to the sequential baseline in Appendix~\ref{appendix:rep_penalty}.

\section{Conclusion}

The experimental results of this work provide important insights for handling multi-objective optimization in reinforcement learning with verifiable rewards for LLMs.

First, multi-objective reinforcement learning with HVC reward (Sequential HVC) was shown to be effective for Pareto-optimal solution exploration compared to single-objective learning (Single-Objective PPA).
At iter=10, the proposed method outperformed Single-Objective PPA\ by 14.4 percentage points in rHV and also showed superiority in circuit diversity.
This result suggests that HVC reward can effectively learn trade-off structures that cannot be captured by conventional approaches that aggregate multiple objective functions into a single scalar value.

Second, it became clear that reward design is extremely important in the proposed generation-based training pipeline.
Sequential TB and Sequential TB+PPA yielded the lowest circuit diversity and, at later iterations, the lowest Win Rate across both base models, indicating that with testbench reward or scalar PPA reward alone, reward hacking occurred where the model simply regenerates previously generated circuits.
This indicates that to promote exploration of new solutions in generation-based training, contributions to the Pareto front must be explicitly incorporated into the reward.

Third, as generation count increases, the superiority of the proposed method expands, and performance improvement continues even for longer exploration (up to iter=10) than training (up to iter=5).
This generalization capability suggests that the proposed method acquires the essential ability of sequential Pareto solution exploration using circuit history, rather than simply overfitting to specific exploration lengths.

\paragraph{Limitations}
This work has several limitations.
First, this work targets only circuit design, and application to other engineering fields such as control engineering is unverified.
Second, this work verified only the setting where the objective function space dimension is 3 (PPA).
Since the proposed method requires computing Hypervolume for each rollout during training, addressing increased computational complexity when the objective function space dimension increases remains a challenge.
Third, due to LLM context length limitations, there is an upper limit on the length of circuit history provided as input.
Therefore, for problems with very many solutions on the Pareto front, it is not possible to provide all history as input, making it difficult to enumerate all solutions.
Fourth, limitations remain regarding data scale: the evaluation benchmark is limited to a single benchmark, RTLLM v2, and generalization to other benchmarks or more diverse circuit specifications remains unverified.
Additionally, the training data used for generation-based training (Sequential TB/Sequential TB+PPA/Sequential HVC) consists of only 78 problems, and verifying effectiveness on larger-scale datasets remains future work.

\section{Acknowledgements}
We sincerely thank Naoaki Okazaki and Youmi Ma of Institute of Science Tokyo for their valuable guidance, insightful discussions, and constructive feedback throughout this research. 
This work was supported by ABCI 3.0, which is provided by AIST and AIST Solutions.

\bibliographystyle{colm2026_conference}
\bibliography{colm2026_conference}

@article{shao2024deepseekmath,
  title   = {{DeepSeekMath}: {P}ushing the {L}imits of {M}athematical {R}easoning in {O}pen {L}anguage {M}odels},
  author  = {Shao, Zhihong and Wang, Peiyi and Zhu, Qihao and Xu, Runxin and Song, Junxiao and Bi, Xiao and Zhang, Haowei and Zhang, Mingchuan and Li, Y. K. and Wu, Y. and Guo, Daya},
  journal = {arXiv preprint arXiv:2402.03300},
  year    = {2024},
  doi     = {10.48550/arXiv.2402.03300}
}

@article{guo2025deepseekr1,
  title   = {{DeepSeek-R1} incentivizes reasoning in {LLMs} through reinforcement learning},
  author  = {Guo, Daya and Yang, Dejian and Zhang, Haowei and others},
  journal = {Nature},
  volume  = {645},
  number  = {8081},
  pages   = {633--638},
  year    = {2025},
  doi     = {10.1038/s41586-025-09422-z}
}

@techreport{openai2024o1,
  author       = {{OpenAI}},
  title        = {{OpenAI} o1 {S}ystem {C}ard},
  institution  = {OpenAI},
  year         = {2024},
  url          = {https://cdn.openai.com/o1-system-card-20241205.pdf},
}

@techreport{openai2025o3,
  author       = {{OpenAI}},
  title        = {{OpenAI} o3 and o4-mini {S}ystem {C}ard},
  institution  = {OpenAI},
  year         = {2025},
  url          = {https://cdn.openai.com/pdf/2221c875-02dc-4789-800b-e7758f3722c1/o3-and-o4-mini-system-card.pdf}
}

@article{yang2025qwen3,
  author       = {Yang, An and Li, Anfeng and Yang, Baosong and Zhang, Beichen and Hui, Binyuan and Zheng, Bo and Yu, Bowen and Gao, Chang and Huang, Chengen and Lv, Chenxu and others},
  title        = {{Qwen3} {T}echnical {R}eport},
  journal      = {arXiv preprint arXiv:2505.09388},
  year         = {2025},
  doi          = {10.48550/arXiv.2505.09388},
  url          = {https://arxiv.org/abs/2505.09388}
}

@article{teamolmo2025olmo3,
  author       = {Allyson Ettinger and Amanda Bertsch and Bailey Kuehl and David W. Graham and David Heineman and Dirk Groeneveld and Faeze Brahman and Finbarr Timbers and Hamish Ivison and others},
  title        = {{OLMo} 3},
  journal      = {arXiv preprint arXiv:2512.13961},
  year         = {2025},
  doi          = {10.48550/arXiv.2512.13961},
  url          = {https://arxiv.org/abs/2512.13961}
}

@inproceedings{wang2025verireason,
  title   = {{VeriReason}: {R}einforcement {L}earning with {T}estbench {F}eedback for {R}easoning-{E}nhanced {V}erilog {G}eneration},
  author  = {Wang, Yiting and Sun, Guoheng and Ye, Wanghao and Qu, Gang and Li, Ang},
  year    = {2026},
  booktitle = {GLSVLSI},
}

@article{chen2025chipseekr1,
  title   = {{ChipSeek-R1}: {G}enerating {H}uman-{S}urpassing {RTL} with {LLM} via {H}ierarchical {R}eward-{D}riven {R}einforcement {L}earning},
  author  = {Chen, Zhirong and Chang, Kaiyan and Li, Zhuolin and He, Xinyang and Chen, Chujie and Li, Cangyuan and Wang, Mengdi and Xu, Haobo and Han, Yinhe and Wang, Ying},
  journal = {arXiv preprint arXiv:2507.04736},
  year    = {2025},
  doi     = {10.48550/arXiv.2507.04736}
}

@inproceedings{hsiao2025buffalo,
  author={Hsiao, Hao-Hsiang and Lu, Yi-Chen and Lim, Sung Kyu and Ren, Haoxing},
  booktitle={2025 IEEE/ACM International Conference On Computer Aided Design (ICCAD)}, 
  title={{BUFFALO}: {PPA}-{C}onfigurable, {LLM}-based {B}uffer {T}ree {G}eneration via {G}roup {R}elative {P}olicy {O}ptimization}, 
  year={2025},
  volume={},
  number={},
  pages={1-9},
  doi={10.1109/ICCAD66269.2025.11240744}}

@inproceedings{liu2024openllmrtl,
  title        = {{OpenLLM-RTL}: {O}pen {D}ataset and {B}enchmark for {LLM}-{A}ided {D}esign {RTL} {G}eneration ({I}nvited)},
  author       = {Liu, Shang and Lu, Yao and Fang, Wenji and Li, Mengming and Xie, Zhiyao},
  booktitle    = {Proceedings of the 43rd IEEE/ACM International Conference on Computer-Aided Design (ICCAD)},
  pages        = {1--9},
  year         = {2024},
  organization = {ACM},
  doi          = {10.1145/3676536.3697118}
}

@inproceedings{teoh2025generalization,
  title     = {On {G}eneralization {A}cross {E}nvironments {I}n {M}ulti-{O}bjective {R}einforcement {L}earning},
  author    = {Teoh, Jayden and Varakantham, Pradeep and Vamplew, Peter},
  booktitle = {The Thirteenth International Conference on Learning Representations},
  year      = {2025},
  note      = {ICLR 2025 Poster}
}

@inproceedings{zhu2025qimengcodevr1,
  title     = {{QiMeng-CodeV-R1}: {R}easoning-{E}nhanced {V}erilog {G}eneration},
  author    = {Zhu, Yaoyu and Huang, Di and Lyu, Hanqi and Zhang, Xiaoyun and Li, Chongxiao and Shi, Wenxuan and Wu, Yutong and Mu, Jianan and Wang, Jinghua and Zhao, Yang and others},
  booktitle = {NeurIPS},
  year      = {2026}
}

@article{IcarusVerilog,
  title   = {{Icarus} {Verilog}: {O}pen-{S}ource {V}erilog {M}ore {T}han a {Y}ear {L}ater},
  author  = {Williams, Stephen and Baxter, Michael},
  journal = {Linux Journal},
  volume  = {2002},
  number  = {99},
  pages   = {3},
  year    = {2002}
}

@inproceedings{OpenROAD,
  title     = {{OpenROAD}: {T}oward a {S}elf-{D}riving, {O}pen-{S}ource {D}igital {L}ayout {I}mplementation {T}ool {C}hain},
  author    = {Ajayi, T. and Blaauw, D. and Chan, T.-B. and Cheng, C.-K. and Chhabria, V. A. and Choo, D. K. and Coltella, M. and Dobre, S. and Dreslinski, R. and Foga{\c{c}}a, M. and others},
  booktitle = {Proceedings of Government Microcircuit Applications and Critical Technology Conference},
  pages     = {1105--1110},
  year      = {2019}
}

@inproceedings{DeepRTL2,
  title     = {{DeepRTL2}: {A} {V}ersatile {M}odel for {RTL}-{R}elated {T}asks},
  author    = {Liu, Yi and Zhang, Hongji and Zhou, Yunhao and Shi, Zhengyuan and Xu, Changran and Xu, Qiang},
  booktitle = {Findings of the Association for Computational Linguistics: ACL 2025},
  pages     = {6485--6500},
  year      = {2025},
  address   = {Vienna, Austria},
  publisher = {Association for Computational Linguistics},
  doi       = {10.18653/v1/2025.findings-acl.336},
}

@article{Dobes2013,
  title   = {Multiobjective {O}ptimization for {E}lectronic {C}ircuit {D}esign in {T}ime and {F}requency {D}omains},
  author  = {Dobe{\v{s}}, Josef and Michal, Jan and Biolkov{\'a}, Viera},
  journal = {Radioengineering},
  volume  = {22},
  number  = {1},
  pages   = {136--152},
  year    = {2013}
}

@article{Visan2022,
  title   = {Automated {C}ircuit {S}izing with {M}ulti-objective {O}ptimization based on {D}ifferential {E}volution and {B}ayesian {I}nference},
  author  = {Vişan, C{\u{a}}t{\u{a}}lin and Pascu, Octavian and St{\u{a}}nescu, Marius and {\c{S}}andru, Elena-Diana and Diaconu, Cristian and Buzo, Andi and Pelz, Georg and Cucu, Horia},
  journal = {Knowledge-Based Systems},
  volume  = {258},
  pages   = {109987},
  year    = {2022},
  doi     = {10.1016/j.knosys.2022.109987}
}

@inproceedings{Liu2013,
  author={Liu, Hung-Yi and Carloni, Luca P.},
  booktitle={2013 50th ACM/EDAC/IEEE Design Automation Conference (DAC)}, 
  title={On learning-based methods for design-space exploration with {H}igh-{L}evel {S}ynthesis}, 
  year={2013},
  volume={},
  number={},
  pages={1-7},
  doi={}}

@inproceedings{Bai2023,
  title     = {Towards a {C}omprehensive {B}enchmark for {H}igh-{L}evel {S}ynthesis {T}argeted to {FPGAs}},
  author    = {Bai, Yunsheng and Sohrabizadeh, Atefeh and Qin, Zongyue and Hu, Ziniu and Sun, Yizhou and Cong, Jason},
  booktitle = {Advances in Neural Information Processing Systems},
  volume    = {36},
  pages     = {45288--45299},
  year      = {2023}
}

@article{AlphaEvolve,
  title   = {{AlphaEvolve}: {A} {C}oding {A}gent for {S}cientific and {A}lgorithmic {D}iscovery},
  author  = {Novikov, Alexander and V{\~u}, Ng{\^a}n and Eisenberger, Marvin and Dupont, Emilien and Huang, Po-Sen and Wagner, Adam Zsolt and Shirobokov, Sergey and Kozlovskii, Borislav and Ruiz, Francisco J. R. and Mehrabian, Abbas and others},
  journal = {arXiv preprint arXiv:2506.13131},
  year    = {2025},
  doi     = {10.48550/arXiv.2506.13131}
}

@article{verilog_history,
author = {Flake, Peter and Moorby, Phil and Golson, Steve and Salz, Arturo and Davidmann, Simon},
title = {{Verilog} {HDL} and its ancestors and descendants},
year = {2020},
issue_date = {June 2020},
publisher = {Association for Computing Machinery (ACM)},
address = {New York, NY, USA},
volume = {4},
number = {HOPL},
doi = {10.1145/3386337},
journal = {Proceedings of the ACM on Programming Languages},
month = jun,
articleno = {87},
numpages = {1-90}
}

@inproceedings{vhdl_overview,
  author={Shahdad, M.},
  booktitle={23rd ACM/IEEE Design Automation Conference}, 
  title={An {O}verview of {VHDL} {L}anguage and {T}echnology}, 
  year={1986},
  volume={},
  number={},
  pages={320-326},
  doi={10.1109/DAC.1986.1586107}}

@inproceedings{WolfGlaser2013Yosys,
author    = {Clifford Wolf and Johann Glaser},
title     = {{Yosys} - {A} {F}ree {V}erilog {S}ynthesis {S}uite},
booktitle = {Proceedings of Austrochip 2013},
year      = {2013}
}

@article{Fortin2012DEAP,
  author  = {F{\'e}lix-Antoine Fortin and Fran{\c c}ois-Michel De Rainville and
             Marc-Andr{\'e} Gardner and Marc Parizeau and Christian Gagn{\'e}},
  title   = {DEAP: Evolutionary Algorithms Made Easy},
  journal = {Journal of Machine Learning Research},
  volume  = {13},
  pages   = {2171--2175},
  year    = {2012}
}

@inproceedings{
  yue2025does,
  title     = {Does Reinforcement Learning Really Incentivize Reasoning Capacity in {LLM}s Beyond the Base Model?},
  author    = {Yang Yue and Zhiqi Chen and Rui Lu and Andrew Zhao and Zhaokai Wang and Shiji Song and Gao Huang},
  booktitle = {NeurIPS},
  year      = {2025},
}

\appendix
\section{Detailed Definition of rHV}
\label{sec:rhv_details}

This section provides the detailed mathematical definition of relative Hypervolume (rHV) and describes the properties of rHV when an objective function degenerates.

\subsection{Mathematical Definition}

Throughout this section, a \emph{solution} $x$ denotes a single generated circuit.
Each solution $x$ is associated with the three-dimensional objective vector $\mathbf{o}(x) = (o_P(x), o_D(x), o_A(x))$ defined in Equation~\ref{eq:objective}, whose components are the golden-normalized power, delay, and area objectives.
We write $f_{i,k}(x)$ for the $k$-th component of this vector evaluated on problem $i$ ($k \in \{1, 2, 3\}$); all three objectives are \emph{maximized}, i.e., larger is better.

Let $x^{(m,i,j,l)}_1, x^{(m,i,j,l)}_2, \dots$ denote the circuits generated by method $m$ on problem $i$ in trial $j$ at iteration $l$.
The cumulative solution set $S_{m,i,j}^{(\ell)}$ collects all circuits generated by method $m$ on problem $i$ in trial $j$ from the first iteration through iteration $\ell$:
\begin{equation}
  S_{m,i,j}^{(\ell)} = \bigcup_{l=1}^{\ell} \left\{ x^{(m,i,j,l)}_1, x^{(m,i,j,l)}_2, \dots \right\}
\end{equation}
This is the per-trial solution set whose Hypervolume the numerator of rHV measures.

To obtain a common reference for problem $i$, we take the union $U_i$ of the cumulative solution sets over all methods $m \in \mathcal{M}$ and all $n_t$ trials at the final iteration $n_{\mathrm{iter}}$:
\begin{equation}
  U_i = \bigcup_{m \in \mathcal{M}} \bigcup_{j=1}^{n_t} S_{m,i,j}^{(n_{\mathrm{iter}})}
\end{equation}

The reference archive $\Pi_i$ is the non-dominated set of this union, i.e., the pseudo Pareto front used as the denominator of rHV:
\begin{equation}
  \Pi_i = \mathrm{ND}(U_i)
\end{equation}

Because the raw PPA values differ in scale across objectives and problems, each objective $k \in \{1, 2, 3\}$ is min--max normalized using the extreme values observed in $U_i$:
\begin{equation}
  \tilde{f}_{i,k}(x) = \frac{f_{i,k}(x) - q^{\mathrm{low}}_{i,k}}{q^{\mathrm{high}}_{i,k} - q^{\mathrm{low}}_{i,k} + \varepsilon}
\end{equation}
where $q^{\mathrm{low}}_{i,k} = \min_{x \in U_i} f_{i,k}(x)$ and $q^{\mathrm{high}}_{i,k} = \max_{x \in U_i} f_{i,k}(x)$ are the per-objective minimum and maximum over the union, and $\varepsilon = 10^{-20}$ avoids division by zero.
For any solution set $Z$, we write $\tilde{Z}$ for the set obtained by mapping every solution $x \in Z$ to its normalized objective vector $(\tilde{f}_{i,1}(x), \tilde{f}_{i,2}(x), \tilde{f}_{i,3}(x))$.
In particular, $\tilde{S}_{m,i,j}^{(\ell)}$ and $\tilde{\Pi}_i$ are the normalized per-trial solution set and the normalized reference archive, respectively.

rHV is then the Hypervolume of the normalized per-trial solution set divided by that of the normalized reference archive:
\begin{equation}
  \mathrm{rHV}_{m,i,j}^{(\ell)} = \frac{\mathrm{HV}\left( \tilde{S}_{m,i,j}^{(\ell)}; \mathbf{r} \right)}{\mathrm{HV}\left( \tilde{\Pi}_i; \mathbf{r} \right)}
\end{equation}
where the reference point is $\mathbf{r} = (-\delta, -\delta, -\delta)$ with $\delta = 10^{-12}$.
The reference point is placed slightly below the origin, rather than at the origin itself, so that rHV remains well-defined and reduces to a 2-dimensional relative Hypervolume when an objective degenerates (Section~\ref{subsec:rhv_degenerate}).

\subsection{Property of rHV When an Objective Degenerates}\label{subsec:rhv_degenerate}

\begin{proposition}
Consider a 3-objective maximization problem with the normalized reference point
\[
\mathbf r=(-\delta,-\delta,-\delta),\qquad \delta>0.
\]
Assume that the third objective is constant across all solutions, and after normalization:
\[
\tilde f_{i,3}(x)=0.
\]
Then, for any point set
\[
P \subset [0,1]^2\times\{0\},
\]
the following holds:
\[
\mathrm{HV}\!\left(P;(-\delta,-\delta,-\delta)\right)
=
\delta\,
\mathrm{HV}\!\left(P_{12};(-\delta,-\delta)\right),
\]
where
\[
P_{12}=\{(u,v)\mid (u,v,0)\in P\}
\]
is the projection onto the first and second objectives.
\end{proposition}

\begin{proof}
By the definition of Hypervolume for maximization problems:
\[
\mathrm{HV}(P;\mathbf r)
=
\lambda_3\!\left(
\bigcup_{p\in P}
[r_1,p_1]\times[r_2,p_2]\times[r_3,p_3]
\right).
\]

Since every point has the form $(u,v,0)\in P$, each point spans a hyperrectangle:
\[
[-\delta,u]\times[-\delta,v]\times[-\delta,0].
\]
Therefore:
\[
\bigcup_{(u,v,0)\in P}
[-\delta,u]\times[-\delta,v]\times[-\delta,0]
=
\left(
\bigcup_{(u,v)\in P_{12}}
[-\delta,u]\times[-\delta,v]
\right)\times[-\delta,0].
\]
By the product property of volume:
\begin{align}
\mathrm{HV}\!\left(P;(-\delta,-\delta,-\delta)\right)
&=
\lambda_2\!\left(
\bigcup_{(u,v)\in P_{12}}
[-\delta,u]\times[-\delta,v]
\right)\cdot\delta \\
&=
\delta\,
\mathrm{HV}\!\left(P_{12};(-\delta,-\delta)\right).
\end{align}
\end{proof}

Applying this result to both the numerator and denominator:
\[
\mathrm{HV}\!\left(\tilde S_{m,i,j}^{(\ell)};\mathbf r\right)
=
\delta\,
\mathrm{HV}\!\left((\tilde S_{m,i,j}^{(\ell)})_{12};(-\delta,-\delta)\right),
\]
\[
\mathrm{HV}\!\left(\tilde \Pi_i;\mathbf r\right)
=
\delta\,
\mathrm{HV}\!\left((\tilde \Pi_i)_{12};(-\delta,-\delta)\right).
\]
Therefore:
\[
\mathrm{rHV}_{m,i,j}^{(\ell)}
=
\frac{
\mathrm{HV}\!\left((\tilde S_{m,i,j}^{(\ell)})_{12};(-\delta,-\delta)\right)
}{
\mathrm{HV}\!\left((\tilde \Pi_i)_{12};(-\delta,-\delta)\right)
}.
\]
Therefore, when the third objective takes the same value across all solutions, the third objective appears only as a common factor in both numerator and denominator, and rHV effectively reduces to a 2-dimensional relative Hypervolume based on the first and second objectives.

\section{PPA Reward}\label{appendix:ppa_reward}

\begin{equation}
  r_{\text{PPA}} = \frac{P_{\text{golden}}}{P_{\text{gen}}} \times \frac{A_{\text{golden}}}{A_{\text{gen}}} \times \frac{D_{\text{golden}}}{D_{\text{gen}} + \epsilon}
\end{equation}
where $P$, $A$, and $D$ denote power, area, and delay, respectively.
The subscript ``golden'' refers to the ground-truth circuit and ``gen'' refers to the generated circuit.
$\epsilon$ is a small constant to avoid division by zero.

\section{Prompt Templates}\label{appendix:prompt}

We show the prompt templates used for multi-objective reinforcement learning below.
\ph{spec} represents the circuit specification, \ph{circuit} represents the generated Verilog circuit, and \ph{ppa} represents PPA values as placeholders.

\subsection{System Prompt}

\begin{promptbox}
You are a helpful assistant. The assistant first thinks about
the reasoning process in the mind and then provides the user
with the answer. The reasoning process and answer are enclosed
within <think> </think> and <answer> </answer> tags,
respectively.

Now the user asks you to write verilog code as shown in <Spec>.
After thinking, enclose the final verilog code in ```verilog```
within <answer> </answer> tags.

Note that circuits satisfying the same specification have
already been obtained, and you need to derive a Pareto-optimal
solution with different PPA values from those circuits. The
circuits obtained so far and their PPA values are listed below
in <# Circuit history>.
\end{promptbox}

\subsection{User Prompt (iteration = 1)}

\begin{promptbox}
## Spec
{spec}

## Circuit history
No previous circuit.
\end{promptbox}

\subsection{User Prompt (iteration $\geq$ 2)}

\begin{promptbox}
## Spec
{spec}

## Circuit history
### previous circuit 1
```verilog
{circuit_1}
```
- PPA: {ppa_1}

### previous circuit 2
```verilog
{circuit_2}
```
- PPA: {ppa_2}

...

### previous circuit i
```verilog
{circuit_i}
```
- PPA: {ppa_i}
\end{promptbox}

\section{Hyperparameters}\label{appendix:hyperparameters}

Table~\ref{tab:hyperparameters} reports the hyperparameters used for training.

\begin{table}[h]
\centering
\caption{Training hyperparameters. Common GRPO settings are shared across all methods on both base models; Max Steps differs by base model.}
\label{tab:hyperparameters}
\begin{tabular}{lcc}
\toprule
\textbf{Hyperparameter} & \textbf{VeriReason-3B} & \textbf{CodeV-R1-Distil-7B} \\
\midrule
Num. GPUs & 4 & 4 \\
GPU & H200 & H200 \\
Batch Size (per device) & 2 & 2 \\
Gradient Accumulation Steps & 4 & 4 \\
Learning Rate & $1.0 \times 10^{-6}$ & $1.0 \times 10^{-6}$ \\
Max Steps & 15,000 & 1,500 \\
Number of Generations & 8 & 8 \\
Temperature & 0.9 & 0.9 \\
KL Penalty ($\beta$) & 0.04 & 0.04 \\
Max Prompt Length & 1,024 & 1,024 \\
Max Completion Length & 8,192 & 8,192 \\
Max History Length per Sample$^{\dagger}$ & 5 & 5 \\
\bottomrule
\end{tabular}
\par\vspace{2pt}
\footnotesize{$^{\dagger}$Applies only to generation-based training methods (Sequential TB / Sequential TB+PPA / Sequential HVC); not used for Single-Objective PPA.}
\end{table}

\section{Additional Analysis}
\label{sec:additional_results}

\subsection{Bootstrap Confidence Intervals}\label{appendix:bootstrap_ci}

To assess statistical reliability, we computed 95\% confidence intervals for the rHV differences between Sequential HVC (Ours) and each baseline at iter=10 using a nested paired bootstrap.
Specifically, we resampled the 45 problems with replacement and, within each sampled problem, resampled the 30 trials with replacement while preserving paired trial comparisons, and repeated this procedure 10{,}000 times.
Table~\ref{tab:rhv_ci} reports bootstrap confidence intervals (95\%) of the rHV difference between Sequential HVC (Ours) and each baseline at iter=10.

\begin{table}[t]
\centering
\caption{Bootstrap confidence intervals (95\%) of the rHV difference between Sequential HVC (Ours) and each baseline at iter=10.}
\label{tab:rhv_ci}
\begin{tabular}{llc}
\toprule
\textbf{Base Model} & \textbf{Comparison} & \textbf{rHV Diff (95\% CI)} \\
\midrule
\multirow{4}{*}{VeriReason-3B} & Ours $-$ Base & $+13.3\ [+2.1,\ +24.8]$ \\
& Ours $-$ Single-Objective PPA & $+14.5\ [+2.5,\ +27.0]$ \\
& Ours $-$ Sequential TB & $+13.2\ [+4.7,\ +22.7]$ \\
& Ours $-$ Sequential TB+PPA & $+10.8\ [+1.2,\ +20.9]$ \\
\midrule
\multirow{4}{*}{CodeV-R1-Distil-7B} & Ours $-$ Base & $+4.8\ [+0.6,\ +9.0]$ \\
& Ours $-$ Single-Objective PPA & $+5.4\ [+1.4,\ +9.9]$ \\
& Ours $-$ Sequential TB & $+9.7\ [+3.5,\ +16.6]$ \\
& Ours $-$ Sequential TB+PPA & $+11.2\ [+4.9,\ +18.3]$ \\
\bottomrule
\end{tabular}
\end{table}

For both VeriReason-3B and CodeV-R1-Distil-7B, the 95\% confidence intervals of the improvements over all baselines do not cross zero, indicating that the gains are stable even when accounting for both problem-level and trial-level variation.

\subsection{Relative Hypervolume with Respect to the Baseline-Only Archive}\label{appendix:rhv_baseline_only}

To quantify the effect of including our own outputs in the reference archive, we conduct an out-of-self reference analysis: we remove all outputs of Sequential HVC (Ours) across all trials and rebuild the reference archive using only the four baseline methods (Base, Single-Objective PPA, Sequential TB, Sequential TB+PPA), and evaluate rHV with respect to this baseline-only archive.
However, for this baseline-only reference archive, the hypervolume of the reference archive becomes extremely small for some problems, so the mean of this ratio-based rHV is dominated by a few outliers.
We therefore compare methods using the median rHV instead of the mean for this additional analysis.
Table~\ref{tab:rhv_baseline_only} reports the median relative Hypervolume (rHV) computed with respect to this baseline-only archive.

\begin{table}[t]
\centering
\caption{Median relative Hypervolume (rHV) computed with respect to the baseline-only archive. Each column shows the median rHV after a given number of sequential generation iterations (\textit{iter}).}
\label{tab:rhv_baseline_only}
\begin{tabular}{llccc}
\toprule
\textbf{Base Model} & \textbf{Method} & \textbf{rHV (iter=1)} & \textbf{rHV (iter=5)} & \textbf{rHV (iter=10)} \\
\midrule
\multirow{5}{*}{VeriReason-3B} & Base & 2.2\% & 6.4\% & 10.3\% \\
& Single-Objective PPA & 0.9\% & 5.3\% & 10.3\% \\
& Sequential TB & 0.0\% & 0.2\% & 0.3\% \\
& Sequential TB+PPA & 0.1\% & 0.9\% & 3.4\% \\
& Sequential HVC (Ours) & \textbf{3.0\%} & \textbf{15.1\%} & \textbf{19.4\%} \\
\midrule
\multirow{5}{*}{CodeV-R1-Distil-7B} & Base & 11.5\% & 23.9\% & 32.1\% \\
& Single-Objective PPA & 14.4\% & 29.2\% & 36.1\% \\
& Sequential TB & 19.8\% & 33.0\% & 33.0\% \\
& Sequential TB+PPA & 17.7\% & 27.1\% & 29.8\% \\
& Sequential HVC (Ours) & \textbf{24.5\%} & \textbf{36.5\%} & \textbf{41.8\%} \\
\bottomrule
\end{tabular}
\end{table}

Sequential HVC (Ours) still achieves the highest median rHV with respect to the baseline-only archive for both the VeriReason-3B and CodeV-R1-Distil-7B bases, confirming that the original conclusion is not merely an artifact of including Sequential HVC in the reference archive.

\subsection{Objective-Space Coverage}\label{appendix:scatter_plot}

To gain a more intuitive understanding of how each method covers the Pareto front, we visualize the solutions generated by each model at test time as two-dimensional projections onto pairs of objective axes.
Figure~\ref{fig:scatter_plot} shows the resulting scatter plots for each method.
Because all three objectives are normalized so that larger values are better (Equation~\ref{eq:objective}), solutions lying toward the upper-right region of each projection are more desirable; a method is therefore preferable when its solutions are both widely spread and pushed toward this region.
The visualization confirms that the proposed method (Sequential HVC) distributes diverse solutions across the desirable upper-right region, complementing the aggregated rHV results reported in the main text.

\begin{figure}[t]
  \centering
  \begin{subfigure}[t]{\linewidth}
    \centering
    \includegraphics[width=0.95\linewidth]{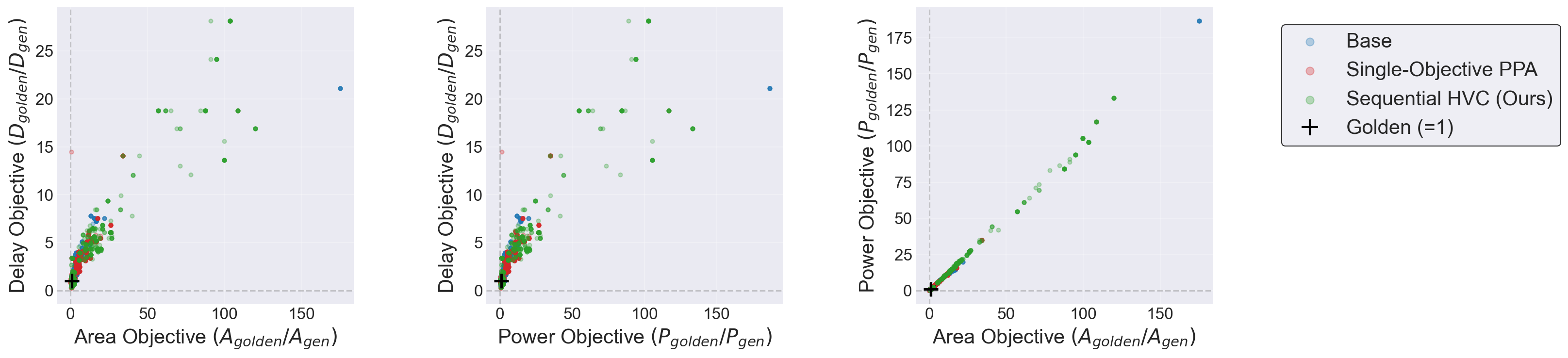}
    \caption{VeriReason-3B}
    \label{fig:scatter_plot_verireason}
  \end{subfigure}
  \\[1em]
  \begin{subfigure}[t]{\linewidth}
    \centering
    \includegraphics[width=0.95\linewidth]{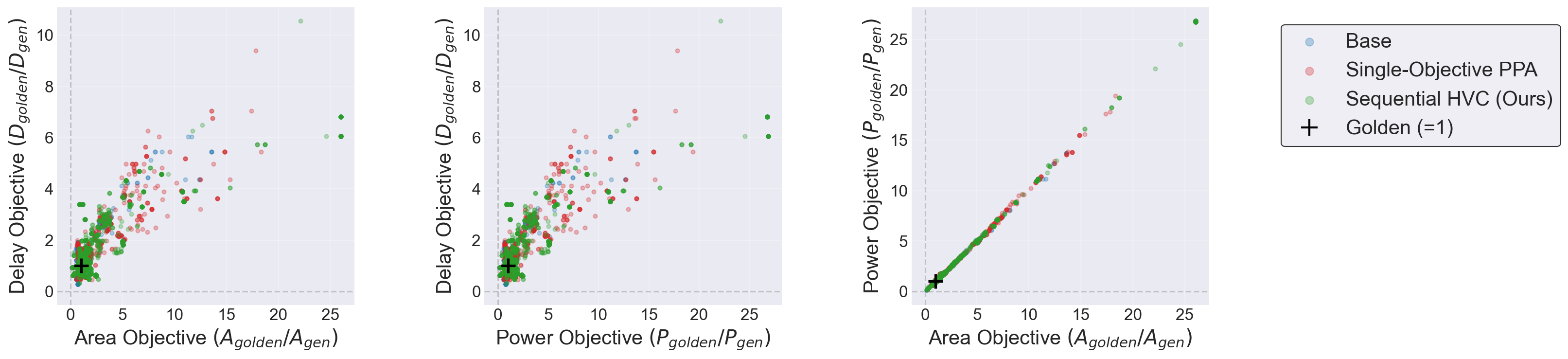}
    \caption{CodeV-R1-Distil-7B}
    \label{fig:scatter_plot_codevr1}
  \end{subfigure}
  \caption{Scatter plots of the generated solutions in the objective space for each base model. Visualizing the solutions in the objective space clarifies how each method covers different trade-off regions of the Pareto front. Sequential HVC (Ours) attains a broader coverage of the objective space than the baselines.}
  \label{fig:scatter_plot}
\end{figure}

\subsection{Repetition Penalty Ablation}\label{appendix:rep_penalty}

The Sequential HVC reward encourages the model to generate circuits that contribute to the Pareto front, which also has the side effect of discouraging duplicate generations.
It is therefore not obvious whether the gains of Sequential HVC come from directly optimizing the Pareto-front contribution or simply from learning to avoid repeated generations and produce unique circuits.
To disentangle these two effects, we run an ablation that adds an explicit repetition penalty to the sequential baseline and test whether avoiding duplicates alone yields a comparable improvement.
Figure~\ref{fig:rep_penalty} reports the Win Rate and circuit diversity of the sequential baseline with and without the repetition penalty.
While the repetition penalty substantially suppresses the generation of duplicate circuits, it still underperforms Sequential HVC in terms of both rHV and Win Rate.
This indicates that avoiding repetition alone is not sufficient for effective Pareto-front exploration, and that directly optimizing the Pareto-front contribution through the HVC reward is crucial.

\begin{figure}[t]
  \centering
  \begin{subfigure}[t]{\linewidth}
    \centering
    \includegraphics[width=\linewidth]{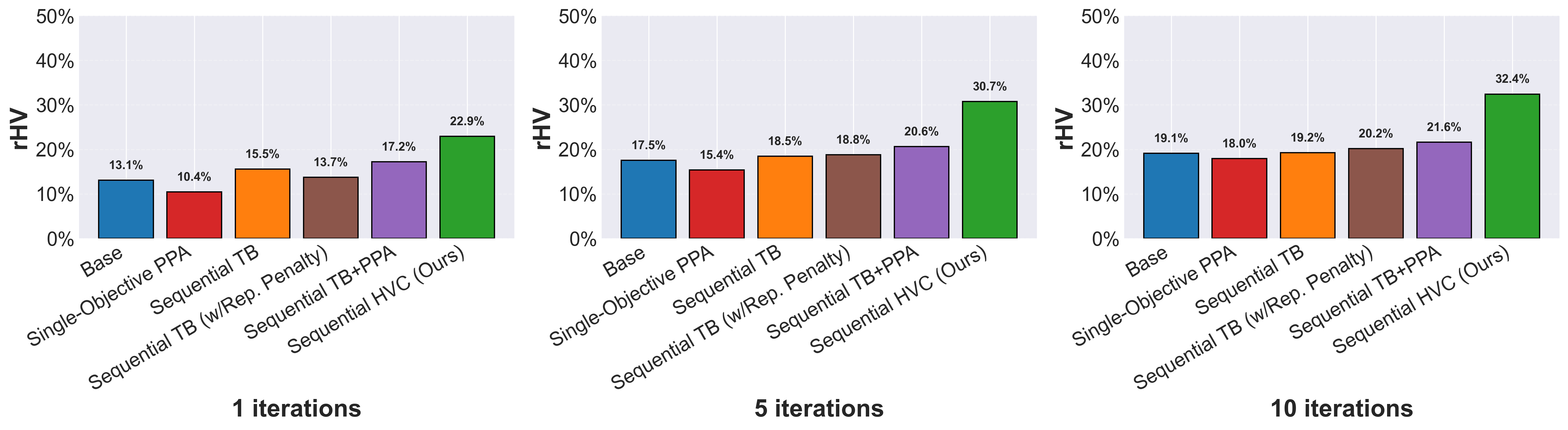}
    \caption{Relative Hypervolume (rHV)}
    \label{fig:rep_penalty_rhv}
  \end{subfigure}
  \\[1em]
  \begin{subfigure}[t]{0.48\linewidth}
    \centering
    \includegraphics[width=\linewidth]{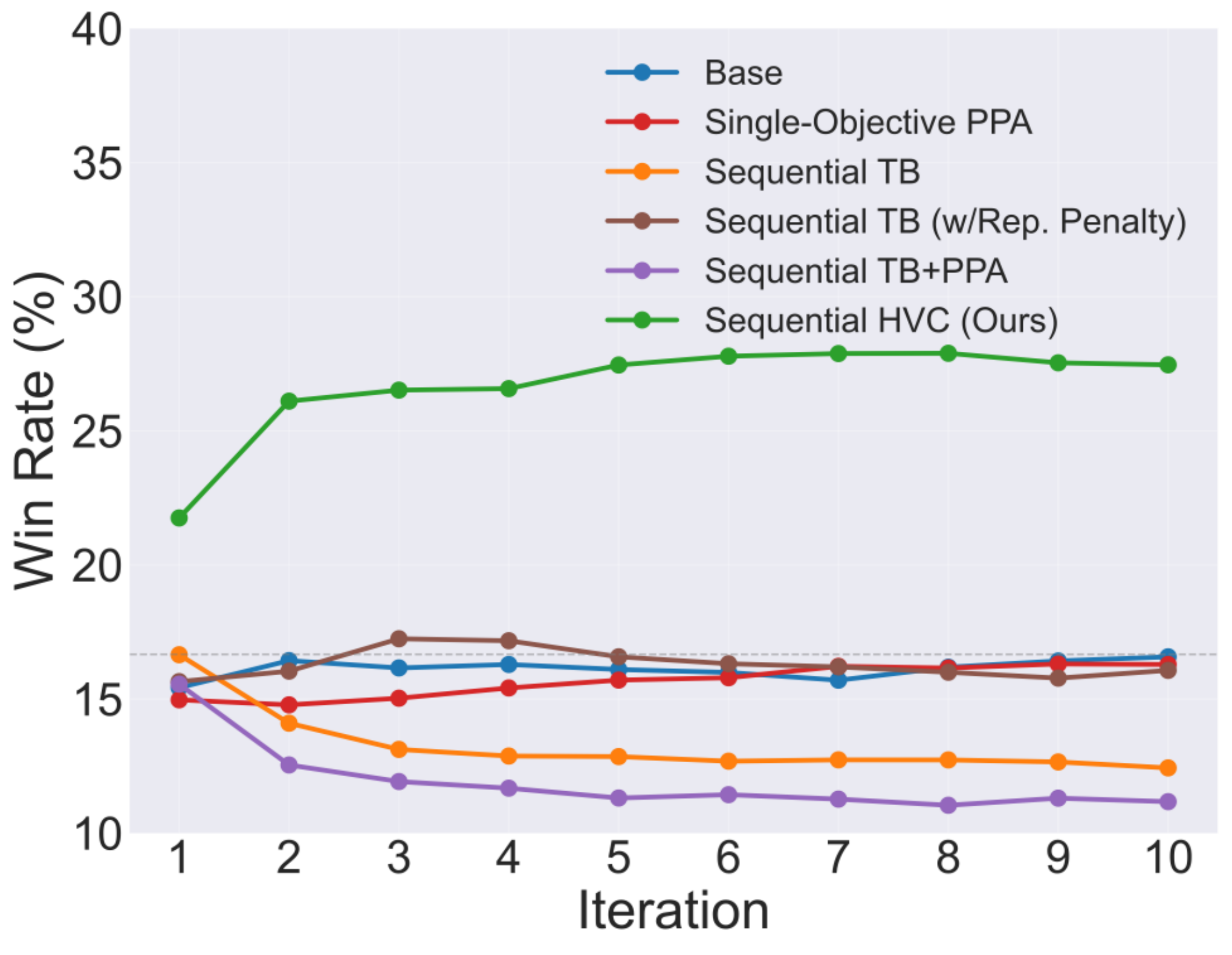}
    \caption{Win Rate}
    \label{fig:rep_penalty_win_rate}
  \end{subfigure}
  \hfill
  \begin{subfigure}[t]{0.48\linewidth}
    \centering
    \includegraphics[width=\linewidth]{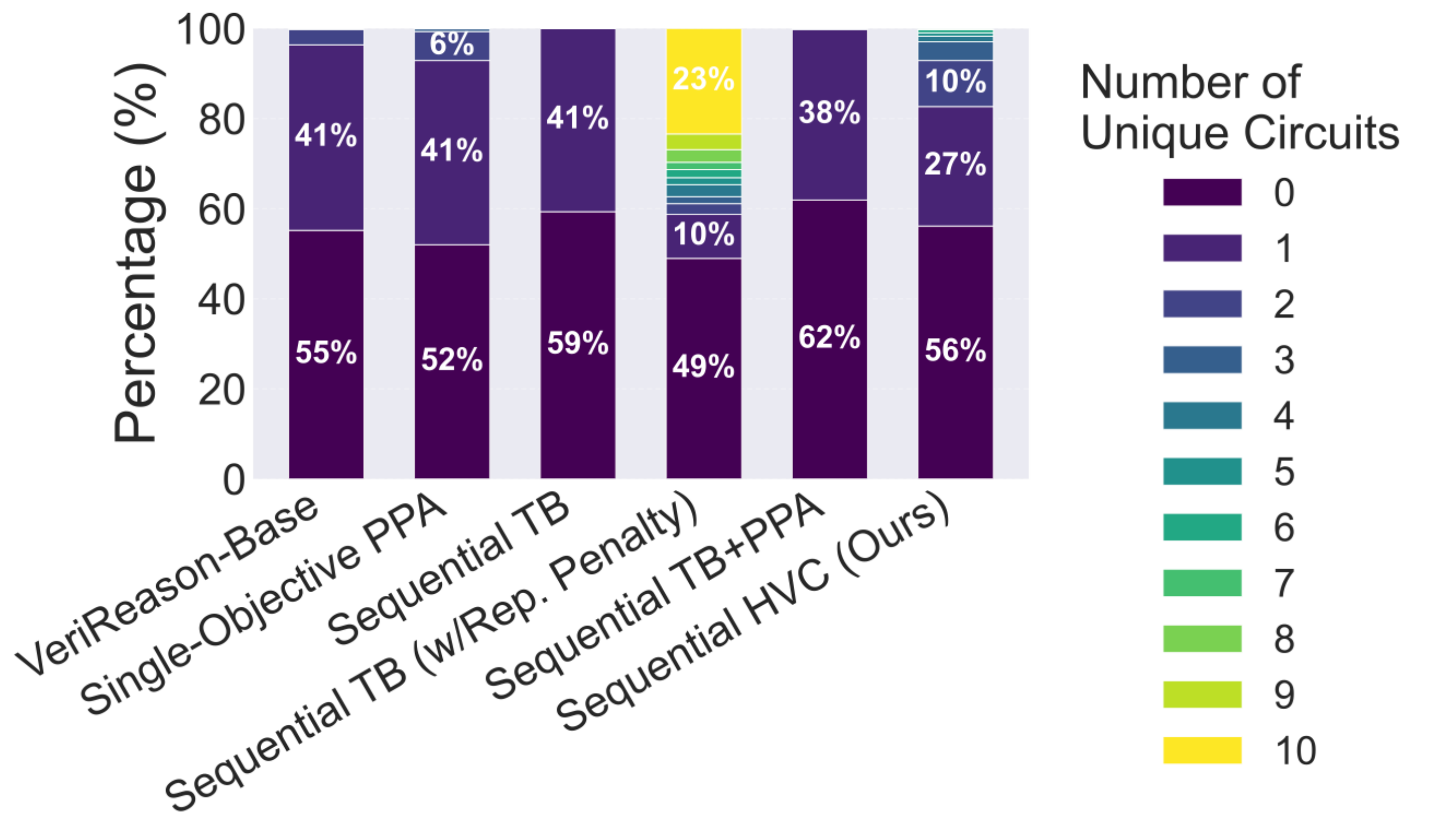}
    \caption{Circuit Diversity}
    \label{fig:rep_penalty_circuit_diversity}
  \end{subfigure}
  \caption{Effect of adding an explicit repetition penalty to the sequential baseline, evaluated in terms of relative Hypervolume (top), Win Rate (bottom left), and circuit diversity (bottom right). Although the repetition penalty reduces duplicate generations, it does not match the improvement of Sequential HVC, indicating that the gains of our method arise from directly optimizing the Pareto-front contribution rather than merely avoiding duplicates.}
  \label{fig:rep_penalty}
\end{figure}

\subsection{Relative Hypervolume of the Union across Trials}\label{appendix:union}

The main results evaluate rHV on a per-trial basis.
To assess the aggregate exploration capability of each method, we additionally evaluate rHV after taking the union of the solutions generated across multiple trials for each method.
Figure~\ref{fig:union} reports the resulting rHV, providing a view of how much of the Pareto front each method can cover when its independent trials are combined.

Figure~\ref{fig:union} arranges the results in a grid of bar charts, with rows corresponding to the number of iterations (1, 5, and 10) and columns corresponding to the number of unioned trials $N_{\mathrm{trials}} \in \{1, 5, 10, 15, 20, 25, 30\}$; each subplot compares the rHV of the five methods as bars.
For both VeriReason-3B and CodeV-R1-Distil-7B, Sequential HVC shows an advantage when $N_{\mathrm{trials}}$ is small.
For the CodeV-R1-Distil-7B experiments, even at $N_{\mathrm{trials}}=30$ Sequential HVC achieves higher rHV than all methods except the Base model, to which it remains comparable, consistent with the phenomenon reported by \citet{yue2025does}, where RL-trained models improve small-$k$ pass@$k$ performance while the base model catches up as $k$ grows.
Overall, the main contribution of our method is not to dominate the Base model under an extremely large archive budget, but to obtain diverse and high-quality circuits efficiently with a smaller inference budget.
This matters for practical use cases where designers want to quickly obtain diverse candidate circuits, and for downstream inference-time search or evolutionary algorithms that benefit from efficient generation of promising initial candidates.

\begin{figure}[p]
  \centering
  \begin{subfigure}[t]{\linewidth}
    \centering
    \includegraphics[height=0.88\textheight]{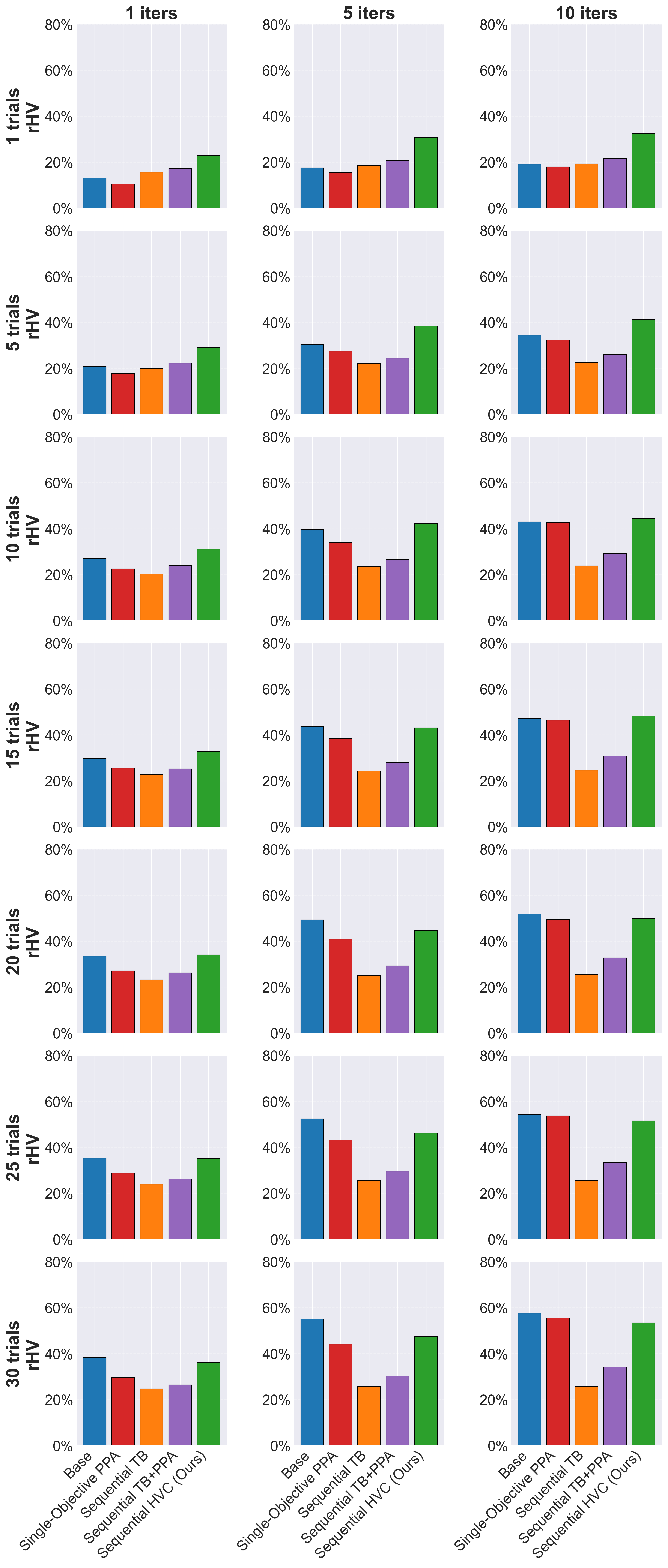}
    \caption{VeriReason-3B}
    \label{fig:union_verireason}
  \end{subfigure}
  \caption{Relative Hypervolume (rHV) evaluated on the union of solutions generated across multiple trials for each method. Aggregating the solutions from independent trials measures the overall Pareto-front coverage attainable by each method. Sequential HVC (Ours) retains its advantage over the baselines for smaller numbers of unioned trials. The results for CodeV-R1-Distil-7B are shown on the following page.}
  \label{fig:union}
\end{figure}

\begin{figure}[p]\ContinuedFloat
  \centering
  \begin{subfigure}[t]{\linewidth}
    \centering
    \includegraphics[height=0.88\textheight]{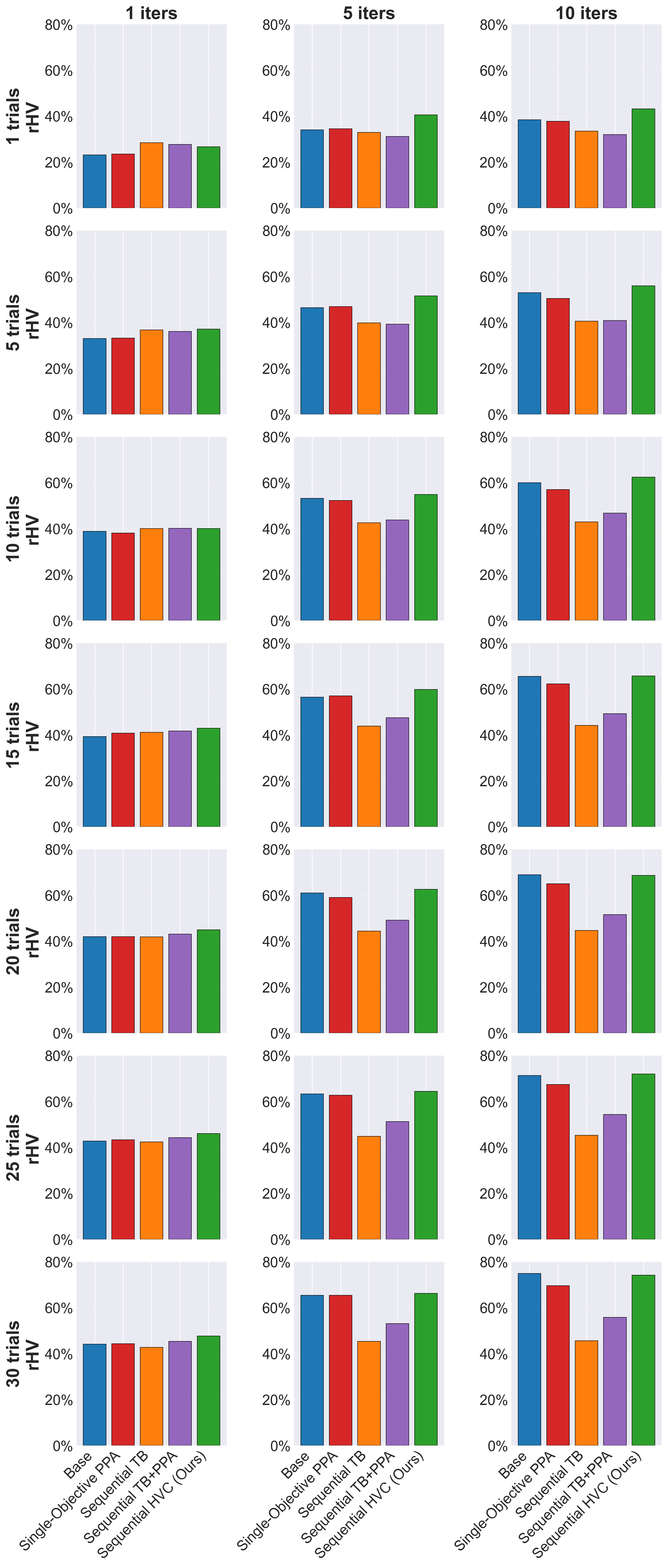}
    \caption{CodeV-R1-Distil-7B}
    \label{fig:union_codevr1}
  \end{subfigure}
  \caption{(continued) Relative Hypervolume (rHV) evaluated on the union of solutions generated across multiple trials for each method.}
\end{figure}

\end{document}